\documentclass[Afour,sageh,times]{sagej}
\usepackage{algorithmicx}
\usepackage{algorithm}
\usepackage[]{algpseudocode}
\algnewcommand\algorithmicinput{\textbf{Input:}}
\algnewcommand\INPUT{\item[\algorithmicinput]}

\usepackage{tikz} 
\usepackage{tkz-graph}
\usepackage{tkz-berge}
	\usetikzlibrary{backgrounds,fit,shapes,snakes,arrows,shapes.geometric,positioning}
	\usetikzlibrary{intersections,patterns,shapes.misc}
	\usetikzlibrary{decorations.pathmorphing}
	\tikzstyle{block} = [rectangle, rounded corners, minimum width=3cm, minimum height=1cm,text centered, draw=black, fill=red!30]
	\tikzstyle{new} = [rectangle, rounded corners, minimum width=1cm, minimum
	height=1cm,text centered, draw=black, fill=blue!10!white, dashed]
	\tikzstyle{arrow} = [thick,->,>=stealth]
	\usetikzlibrary{calc, quotes}

\usepackage{fontawesome}
\DeclareFontFamily{OT1}{pzc}{}
\DeclareFontShape{OT1}{pzc}{m}{it}{<-> s * [1.200] pzcmi7t}{}
\DeclareMathAlphabet{\mathpzc}{OT1}{pzc}{m}{it}

\AtBeginDocument{%
  \DeclareMathAlphabet\PazoBB{U}{fplmbb}{m}{n}%
}

\usepackage{dsfont}
\usepackage{microtype}
\usepackage{diagbox}
\usepackage{multirow}
\usepackage{dashbox}

\usepackage{sidecap}

\newcommand{\TotalUni}{\textbf{TU}}
\newcommand{\TotalNon}{\textbf{TN}}
\newcommand{\IndivUni}{\textbf{IU}}
\newcommand{\Gcal}{\Gall}

\newcommand{\GG}{{G}}

\newcommand{\VV}{{V}}

\newcommand{\EE}{{E}}

\newcommand{\zero}{\mathbf{0}}

\newcommand{\HH}{\mathbf{H}}

\newcommand{\bell}{\boldsymbol{\ell} }

\newcommand{\ppp}{\boldsymbol\pi}
\newcommand{\AAA}{\mathbf{A}}

\newcommand{\Acal}{\mathcal{A}}

\newcommand{\Bcal}{\mathcal{B}}

\newcommand{\Wcal}{\mathcal{W}}
\newcommand{\Ccal}{\mathcal{C}}

\newcommand{\JJ}{\mathbf{{J}}}

\DeclareMathAlphabet\mathbfcal{OMS}{cmsy}{b}{n}

\newcommand{\Rset}{\mathbb{R}}

\DeclareMathOperator*{\argmax}{arg\,max}

\newcommand{\Ecal}{{E}}
\newcommand{\Vcal}{{V}}
\newcommand{\Ucal}{\mathcal{U}}

\newcommand{\kcomp}{k}
\newcommand{\kcomm}{b}

\let\oldsim\sim
\renewcommand{\sim}{\overset{e}{\oldsim}}

\newcommand{\fe}{f}
\newcommand{\he}{h}

\newcommand{\fv}{g}
\newcommand{\fm}{g}
\newcommand{\edg}{\mathsf{edges}}
\newcommand{\OPTe}{\text{\text{OPT}$_1$}}

\newcommand{\Vgrd}{\VV_\text{grd}}
\newcommand{\Egrd}{\Ecal_\text{grd}}
\newcommand{\ndim}{D_\text{NetVLAD}}
\newcommand{\pthresh}{p_\text{x}}

\definecolor{kkGreen}{RGB}{201,232,206}
\definecolor{kkRed}{RGB}{255,196,215}
\definecolor{kkBlue}{RGB}{214,226,255}
\newcommand{\Gall}{\GG_\mathsf{x}}
\newcommand{\Vall}{\VV_\mathsf{x}}
\newcommand{\Eall}{\EE_\mathsf{x}}

\newtheorem{definition}{Definition}
\newtheorem{problem}{Problem}
\newtheorem{remark}{Remark}
\newtheorem{theorem}{Theorem}
\newtheorem{lemma}{Lemma}

\usepackage[framemethod=tikz]{mdframed}
\usepackage{moreverb,url}
\RequirePackage{amsmath}
\usepackage{url}
\usepackage{caption}
\usepackage{mathtools}
\usepackage[justification=centering]{subcaption}
\usepackage{color}
\usepackage{epstopdf}
\usepackage{float}
\graphicspath{{figures/}}

\newcommand\BibTeX{{\rmfamily B\kern-.05em \textsc{i\kern-.025em b}\kern-.08em
T\kern-.1667em\lower.7ex\hbox{E}\kern-.125emX}}

\def\volumeyear{2019}

\usepackage{secdot}
	\sectiondot{subsection}
	\sectiondot{subsubsection}

\begin{document}

\runninghead{Tian et al.}

\title{A Resource-Aware Approach to Collaborative Loop Closure Detection with Provable Performance Guarantees}

\author{Yulun Tian\affilnum{1}, Kasra Khosoussi\affilnum{1}, and Jonathan P.\ How\affilnum{1}}

\affiliation{\affilnum{1}Massachusetts Institute of Technology, Cambridge, MA, USA}

\corrauth{Yulun Tian}

\email{yulun@mit.edu}

\begin{abstract}

This paper presents resource-aware
algorithms for distributed inter-robot loop closure detection for applications
such as collaborative
simultaneous localization and mapping (CSLAM) and distributed image retrieval. 
In real-world scenarios,
this process is resource-intensive as it involves exchanging many observations and
geometrically verifying a large number of potential matches. This poses severe challenges
for small-size and low-cost robots with various operational and
resource constraints that limit, e.g., energy consumption, communication
bandwidth, and computation capacity. 
This paper proposes a
framework in which robots first exchange
compact queries to identify a set of potential loop closures. We then
seek to select a subset of potential inter-robot loop closures for geometric
verification that maximizes a
monotone submodular performance metric without exceeding budgets on computation (number of
geometric verifications) and
communication (amount of exchanged data for geometric verification). We demonstrate that this problem is in general NP-hard,
and present efficient approximation algorithms with provable performance
guarantees. The proposed framework is extensively evaluated on real and
synthetic datasets.
A natural convex relaxation scheme is also presented to certify the near-optimal performance of the
proposed framework in practice.
\end{abstract}

\maketitle

\section{Introduction}

Inter-robot loop closures tie individual trajectories and maps together, and
allow spatial information to flow from one robot to the entire team.
Finding these inter-robot measurements is a crucial problem in
{collaborative} simultaneous
localization and mapping (CSLAM) and multi-robot navigation 
in GPS-denied environments.
This process requires 
exchanging observations, 
identifying potential matches, and verifying these potential matches
for spatial consistency.
Real-world scenarios often involve long-term missions in environments with high
rate of perceptual aliasing. These challenges make inter-robot loop closure detection
a resource-intensive process with a rapidly growing search space.
This task is especially challenging for the prevalent small-size robotic platforms
that are subject to various operational constraints which ultimately
limit, e.g., energy consumption, communication bandwidth, and computation capacity.
Designing 
{resource-efficient} frameworks
for inter-robot loop closure detection
thus is crucial and
has sparked a growing interest as evident by the recent literature; 
see, e.g.,
\citep{CieslewskiChoudhary17,Giamou18_ICRA,van2019collaborative,khosoussi2019reliable}.

Existing solutions 
can be categorized 
into centralized and decentralized approaches.
Centralized approaches circumvent the onboard resource constraints
by outsourcing the task of inter-robot loop closure detection
to an offboard unit with sufficient resources; see, e.g.,
\citep{schmuck2018ccm}.
These solutions, however, are limited to applications where a sufficiently capable
central node exists and is accessible via a reliable communication channel. 
Furthermore, in these solutions robots need to transmit all observations to the central node
which is inefficient considering the fact 
that only a small fraction of these observations constitute inter-robot loop closures.
To address such limitations,
recent works have
investigated efficient decentralized schemes
in which robots exchange lightweight queries 
and 
\emph{collaboratively} discover and verify inter-robot loop closures;
see, e.g.,
\citep{cieslewski2017efficient,Giamou18_ICRA,CieslewskiChoudhary17} and references therein.

Despite these crucial efforts, an important problem remains to be addressed:
even with the most resource-efficient scheme, 
completing the loop closure detection task
may still be infeasible given the allocated budgets on, e.g., communication bandwidth
and computation capacity. 
In other words, while resource-efficiency is necessary, it is not
sufficient. To address this issue, robots must be aware of such budget constraints, and
further must be able to seamlessly adapt their behaviour to these constraints by intelligently utilizing resources available onboard to complete
their tasks to the best of their ability.

In this work, we present such a \emph{resource-aware} approach to the distributed
inter-robot loop closure detection problem.  
At the core of our approach, the team seeks an exchange-and-verification plan that maximizes a
performance metric (e.g., expected number of discovered loop closures) subject
to budgets on communication (size of transmission) and
computation (number of geometric verifications).
In words, such a plan determines
(i) ``who must share what with whom'' and (ii) ``which subset of potential
inter-robot loop closures should be tested for spatial consistency''.
Special cases of this problem have been studied in recent works
in the context of
measurement selection for SLAM,
where robots are
\emph{only} subject to communication \citep{tian18} or computation
\citep{khosoussi2019reliable,carlone2017attention} budgets.
In general, this problem is NP-hard as it generalizes the maximum coverage problem.
This work builds upon classical results in submodular maximization \citep{krauseSurvey}
and presents efficient
approximation algorithms with provable performance guarantees for
maximizing monotone submodular performance metrics under budgeted
resources.
The performance of the proposed approach is extensively validated using both
real-world and simulated datasets, and the near-optimality of the proposed solution is
demonstrated empirically using a post-hoc certification scheme based on
convex relaxation.

\subsection*{Contributions}
The main contributions of this work are two-fold: 
\begin{enumerate}
  \item A generic sensor-agnostic framework for collaborative loop
	closure detection-and-verification under budgeted communication and computation.
  \item Greedy approximation algorithms with performance guarantees
	for near-optimal planning of data exchange and candidate verification under various
	communication and computation resource constraints models.
\end{enumerate}

An early version of this work was presented at WAFR 2018 (Workshop on the Algorithmic
Foundations of Robotics) \citep{Tian18_WAFR}.
This paper improves \citep{Tian18_WAFR} and our earlier work \citep{tian18}
along the following axes:
\begin{enumerate}
  \item New approximation algorithms with provable performance guarantees in new
	resource constraint regimes with pairwise and individual computational
	budgets 
	(Section~\ref{sec:individualBudgets}).
  \item A data-driven (as opposed to hand-engineered) scheme based on logistic
	regression for learning edge
	weights (i.e., loop closure probabilities) in the exchange graph from
	NetVLAD vectors (Section~\ref{sec:exchange}).
  \item Exploring the trade-off between the size and fidelity of metadata for a new
	metadata representation, namely NetVLAD which offers additional flexibility
	compared to the	previously used bag-of-visual-words models 
	(Section~\ref{sec:metadata_experiments}).
  \item Extended experimental and numerical analysis (Section~\ref{sec:experiments}).
\end{enumerate}

\subsection*{Outline}
We review related works in Section~\ref{sec:relatedworks}.
Section~\ref{sec:overview} presents an overview of the proposed framework and
formulates the main optimization problem. Our main approximation algorithms
are then presented in Sections~\ref{sec:modular} and \ref{sec:submodular} for modular and
submodular objectives, respectively. The proposed framework is experimentally
evaluated in Section~\ref{sec:experiments}. We then conclude the paper in
Section~\ref{sec:conclusion}. Finally, Appendix~\ref{sec:app} provides the proofs.

\subsection*{General Notation}
Bold lower-case and upper-case letters are generally reserved for vectors and matrices, respectively. 
Union of disjoint sets $\Acal_i$'s
is denoted by $\Acal_1 \uplus \Acal_2 \uplus \cdots \uplus \Acal_n$. For any 
subset $\VV$ of vertices in a given graph, $\edg(\VV)$ represents the set of all
edges incident to at least a vertex in $\VV$. Finally, $[n] \triangleq
\{1,2,\ldots,n\}$ for any $n \in \mathbb{N}$.

\section{Related Works}
\label{sec:relatedworks}

Resource-efficient CSLAM in general
\citep{choudhary2017,paull2016unified}, and data-efficient
distributed inter-robot loop closure detection (CSLAM \emph{front-end})
in particular \citep{Giamou18_ICRA,CieslewskiChoudhary17,cieslewski2017efficient,tian18,choudhary2017}
have been active areas of research in recent years.  
This paper focuses on CSLAM front-ends; see, e.g.,
\citep{choudhary2017,paull2015communication,paull2016unified} and the
survey paper by
\cite{saeedi2016multiple} for a discussion of modern resource-efficient CSLAM back-ends.

Modern (visual) appearance-based loop-closure detection techniques
commonly used in pose-graph SLAM
and image retrieval consist of two steps:
\begin{enumerate}
  \item \label{step1}
  	\emph{Place recognition}, 
    during which a compact representation of images such as bag of visual words (BoW)
	\citep{sivic2003video}, vector of locally aggregated descriptors (VLAD) \citep{jegou2010aggregating}, or NetVLAD
	\citep{arandjelovic2016netvlad} is used to efficiently search for
	potential matches for a given query image.
  \item \label{step2}
    \emph{Geometric verification} \citep{philbin2007object},
	which involves matching local image keypoints (e.g., with RANSAC iterations) to prune spatially
	inconsistent potential matches.
\end{enumerate}

In multirobot scenarios such as CSLAM front-end, images are collected by and
initially stored on different robots. Communication is thus needed to establish inter-robot loop
closures. In a centralized solution, this task is outsourced to a ``central
server'' with
sufficient resources (e.g.,
ground station); see, e.g., \citep{schmuck2018ccm}. Each robot thus needs to send every keypoint extracted from
its keyframes (a ``representative'' subset of all frames) to the central server for loop closure detection. 
Such a solution has two major resource efficiency and operational drawbacks: (i) to discover new
inter-robot loop
closures, robots need to maintain
a connection to a central
node that is capable of processing all data; (ii) robots must 
blindly communicate \emph{all} of their data (i.e., keypoints) despite the fact
that, in practice, only a small fraction of them correspond to loop closures.
These shortcomings can be addressed by resource-efficient distributed
paradigms in which robots \emph{collaboratively} discover inter-robot loop
closures.

\citet{CieslewskiS17,cieslewski2017efficient} and \citet{CieslewskiChoudhary17} 
propose effective heuristics to reduce data transmission 
during distributed visual place recognition (step~\ref{step1}).
Specifically, in \citep{cieslewski2017efficient} visual words are preassigned to
robots. Each query is then split and each component is sent to
the robot that is responsible for the corresponding word.
In \citep{CieslewskiS17,CieslewskiChoudhary17}, a similar idea is proposed 
based on clustering NetVLAD descriptors in a training set.
The learned cluster centers are then preassigned to robots
and each NetVLAD query is only sent to the
robot that has been assigned the closest cluster center.
In these works, for each query image only the best admissible match is
geometrically verified.
Although the abovementioned techniques are 
highly relevant to distributed inter-robot loop closure detection, they are
orthogonal to our work as 
we focus on the second stage in the pipeline (geometric verification).
Nonetheless, as it will become more clear
shortly in
Section~\ref{sec:overview}, these ideas can be used alongside ours to improve the overall efficiency of CSLAM front-end.

The work by \citet{Giamou18_ICRA} complements the abovementioned line of
work by focusing on 
resource-efficiency during distributed geometric verification (step~\ref{step2}).
Efficiency during this stage is particularly important as robots 
need to exchange \emph{full} observations (i.e., image keypoints)
whose size is typically between 
$100$-$200$
times the size of compact queries,
e.g., 
when NetVLAD is used.
Rather than immediately exchanging the keypoints
for geometric verification,
in \citep{Giamou18_ICRA} an \emph{exchange graph} is formed in which nodes
correspond to keyframes and edges represent potential matches discovered by
querying compact representations such as BoW or NetVLAD. Robots then seek an
exchange policy with minimum data exchange such that every potential
match can be verified by at least one of the two associated robots.  By
exploiting the structure of the exchange graph, this approach can reduce the
amount of data transmission for geometric verification.

Verifying all potential matches may not be always possible due to the limited
nature of mission-critical resources available onboard. In \citep{tian18}, we
investigate the budgeted data exchange problem where robots have to find
inter-robot loop closures under a communication budget. 
In such situations, one needs a performance metric 
to quantify the \emph{value} of data exchanged for verifying a particular subset
of potential matches.
Maximizing such a performance metric under a communication budget thus guides
the robots to prioritize their transmissions. This problem is in general NP-hard.
In \citep{tian18} we consider a class of monotone submodular performance
metrics and provide provably near-optimal approximation
algorithms for maximizing such functions subject to a
communication budget.

In an early version of the this work, we extend
the formulation in \citep{tian18} to scenarios where, in addition to a communication budget, robots are also subject to a
computational budget \citep{Tian18_WAFR}. Computational budget puts a limit on the number of geometric
verifications and, subsequently, the number of loop closures that will be added to the pose
graph which allows us to control the cost of CSLAM back-end. Similar computational budgets were considered before in
\citep{kasra16wafr,khosoussi2019reliable,carlone2017attention} for selecting
informative measurements in SLAM.
The present work extends \citep{Tian18_WAFR} and provides approximation algorithms with performance
guarantees for inter-robot loop closure detection under both communication and
computational budgets.

In addition to the line of work discussed above, alternative ideas have been
proposed to improve the resource-intensive nature of inter-robot loop closure
detection in specific settings. In particular, 
\citet{choudhary2017} propose an object-based SLAM framework that circumvents the resource constraints by 
compressing sensory observations with high-level semantic
labels. Such an approach would be effective only if the environment is filled
with known objects. Similarly, \citet{van2019collaborative} propose
low-level feature coding and compression schemes for ORB
to reduce data transmission during collaborative mapping.
The proposed approach, however, depends on the 
specific choice of feature representation.

\section{Proposed Framework}
\label{sec:overview}
In this section, we provide an overview of our distributed inter-robot loop closure
detection framework. 
In CSLAM, robots can initiate a search for inter-robot loop closures during their occasional (preplanned
or not) \emph{rendezvous}. We assume that during a ``rendezvous'', robots can 
communicate with some of their peers in close proximity owing to the
broadcast nature of wireless medium.
\begin{definition}[$r$-rendezvous]
\label{def:rendezvous}
\normalfont
An $r$-\emph{rendezvous} ($r \geq 2$) refers to the situation where 
$r$ robots are positioned such that each of them will receive the data broadcasted by
any other robot in that group.
\end{definition}

Each robot arrives at the rendezvous 
with a collection of observations (e.g., images or laser scans) acquired throughout its
mission at different times and locations. 
Our goal is to discover loop closures (associations) between observations owned by different robots.
A \emph{distributed} approach to this problem divides
the burden of the task among the robots,
and thus enjoys several advantages
over centralized schemes such as
reduced data transmission and extra flexibility; see, e.g.,
\citep{Giamou18_ICRA,cieslewski2017efficient}. In these works, robots
first perform distributed \emph{place recognition}
by exchanging a compact representation of their observations (hereafter, \emph{metadata})
in the form of full-image descriptors such as BoW, VLAD, or NetVLAD
vectors.
Alternatively,
spatial clues (i.e., estimated location with uncertainty) can be used 
if a common reference frame is already established.
These compact queries can help robots to efficiently identify a set of
\emph{potential} inter-robot loop closures.
It should be noted that although these queries are quite compact relative to the size
of a full observation (i.e., size of all keypoints in a keyframe), robots
still need to exchange many (i.e., one for each keyframe) of these vectors and search within them to find
potential matches.
In \citep{Giamou18_ICRA} one of the robots (or a central node) is assumed to collect all
metadata and conduct the search. Alternatively, when this is not feasible, one
can turn to lossy (yet effective) alternative procedures for
collaborative search 
such as \citep{CieslewskiChoudhary17,CieslewskiS17,cieslewski2017efficient} to reduce
the burden of the search for potential loop closures.

After the initial phase of place recognition, the identified potential matches are tested for
spatial consistency \citep{philbin2007object} 
during
\emph{geometric verification}.
This typically involves 
RANSAC iterations \citep{FischlerRANSAC}
over keypoint correspondences in the two associated frames; see, e.g., \citep{murORB2}. 
A potential match passes this test if
the number of spatially consistent keypoint correspondences
is above a threshold. These inliers then provide a relative transformation
between the corresponding poses which will be used by the CSLAM back-end.
Note that to collaboratively verify a potential loop closure,
at least one of the associated robots must share its observation (e.g., image keypoints)
with the other robot.

Although each geometric verification can be performed efficiently, with a
growing problem size in long-term missions and under high perceptual
ambiguity, geometric verification can still become the computational bottleneck of
the entire system; see \citep{heinly2015,Raguram2012BMVC}. 
In addition, 
the large amount of \emph{full} observation exchange required for geometric verification
may exceed the allocated/available communication budget (e.g., due to limited energy, bandwidth, and/or
time). To address both issues, in the
rest of this section we describe a general framework for \emph{focusing} the
available computation and communication resources on verifying the most ``informative'' subset of potential
loop closures. An outline of the proposed approach is as follows:
\begin{enumerate}
  \item Robots exchange compact queries (\emph{metadata}) for their
	observations and use them to identify a set of \emph{potential} inter-robot
	loop closures; see, e.g.,
	\citep{cieslewski2017efficient,CieslewskiChoudhary17,Giamou18_ICRA}.
  \item One robot forms the \emph{exchange
	graph} (Definition~\ref{def:exchange_graph} in Section~\ref{sec:exchange}),
	and approximately solves Problem~\ref{prob:codesign} (Sections~\ref{sec:exchange}-\ref{sec:problem})
	using algorithms presented in Sections~\ref{sec:modular} and \ref{sec:submodular}.
	This process determines which full observation (image keypoints)
	needs to be shared with the team and which potential loop
	closures are worthy of being tested for geometric verification based on the
	allocated communication and computation budgets. 
  \item According to the solution obtained in the previous step,
  robots exchange their observations
  and collaboratively verify the selected subset of potential inter-robot loop closures.
\end{enumerate}

\subsection{Exchange Graph}
\label{sec:exchange}
Rather than immediately exchanging observations for an identified potential match
\citep{cieslewski2017efficient,CieslewskiChoudhary17,CieslewskiS17}, in this
work we first form the \emph{exchange graph} induced by the set of potential
matches (see Figure \ref{fig:gex} for a simple exchange graph with $r = 3$).

\begin{definition}[Exchange Graph \citep{Giamou18_ICRA,tian18}]
  \label{def:exchange_graph}
  \normalfont
  Consider an $r$-rendezvous.
  An exchange graph between $r$ robots is a simple
  undirected $r$-partite graph $\Gall = (\Vall,\Eall)$ where each vertex $v
  \in \Vall$ corresponds to an observation collected by one robot at a
  particular time. The vertex set can be partitioned into $r$
  (self-independent) sets $\Vall = \VV_1 \uplus \cdots \uplus \VV_r$.  Each
  edge $\{u,v\} \in \Eall$ denotes a \emph{potential} inter-robot loop
  closure identified by comparing the corresponding metadata (here, $u \in \VV_i$ and $v \in \VV_j$). 
  $\Gall$ is endowed with vertex and edge weights $w : \Vall \to \Rset_{>0}$ and $p : \Eall
\to [0,1]$ that quantify the size of each observation (e.g.,
bytes, number of keypoints in a keyframe, etc), and the probability that an edge corresponds to a true loop closure, respectively. 
\end{definition}

Similar to \citep{CieslewskiChoudhary17,CieslewskiS17}, we choose NetVLAD
\citep{arandjelovic2016netvlad} 
as our metadata representation. 
Given an input image, NetVLAD uses a neural network to extract a normalized vector as the corresponding full-image descriptor.
During place recognition,
we compute the Euclidean distance $d(e)$
between NetVLAD vectors extracted from two keyframes stored on two robots.
This distance is then mapped to an
estimated probability $p(e)$
based on a logistic regression model.
Specifically, we assume a posterior probability of the form,
\begin{equation}
  p(e) = \bigg[1+\exp(-\beta_1 \cdot d(e) - \beta_0)\bigg]^{-1}
\label{eq:logistic}
\end{equation}
where $\beta_0$ and $\beta_1$ are learned offline on a training dataset.
Note that, alternatively, one can also use other features (based on, e.g.,
components of the NetVLAD vectors) in addition
to the Euclidean distance in the logistic regression model. 
Nevertheless,
this does not yield significant performance improvement in our experiments (Section~\ref{sec:experiments}).
If the resulting
probability $p(e)$ is higher than a threshold $\pthresh \in [0,1]$, 
the corresponding pair of keyframes is
considered to be a potential loop closure and is added to the exchange graph as
an edge weighted by its probability.
It is worth noting that based on \eqref{eq:logistic},
enforcing a
threshold on 
$p(e)$
is equivalent to enforcing a (different but unique) threshold on the original distance $d(e)$. 
Choosing a threshold
directly for distances, however, lacks interpretability. Furthermore, the
estimated probabilities are used in this work to incentivise
the robots to select more promising potential matches
(see Section~\ref{sec:objectives}).

Note that the ``quality'' of an exchange graph (e.g., in terms of
precision-recall for a fixed probability threshold) is mainly determined by the
fidelity of the
metadata. Intuitively, as the granularity of metadata
increases, one expects that a higher percentage of resulting potential matches pass
the geometric verification step. 
This, however, comes at the cost of higher data
transmission during place recognition. 
In our proposed framework, 
the dimension of NetVLAD vectors (hereafter, $\ndim$),
which in the pre-trained models can be tuned up to 4096,
serves as a ``knob'' 
which allows us to explore the
trade-off between granularity and communication cost of query and verification
steps; see Section~\ref{sec:metadata_experiments} for quantitative results.

\begin{figure*}[t]
	\centering
	\begin{subfigure}[t]{0.29\textwidth}
		\centering
		\begin{tikzpicture}[scale=1]
		\tikzstyle{vertex}=[circle,fill,scale=0.4,draw]
		\tikzstyle{special vertex}=[circle,fill=red,scale=0.4,draw]
		\tikzstyle{square vertex}=[rectangle,fill,scale=0.5,draw]
		\tikzstyle{diamond vertex}=[regular polygon,regular polygon
		sides=3,rotate=45,fill,scale=0.3,draw]
		\node[vertex] at (1,2.598076) (a1) {};
		\node[vertex] at (1.5,2.598076) (a2) {};
		\node[vertex] at (2,2.598076) (a3) {};
		\node[vertex] at (0.25,1.299037) (b1) {};
		\node[vertex] at (0.5,0.866025) (b2) {};
		\node[vertex] at (0.75, 0.433013) (b3) {};
		\node[vertex] at (2.25,0.433013) (c1) {};
		\node[vertex] at (2.5,0.866025) (c2) {};
		\node[vertex] at (2.75,1.299037) (c3) {};
		\node (r1) at (1.5,3.098076) [] {\faAndroid$_{_1}$};
		\node (r2) at (0.1464466,0.51247) [] {\faAndroid$_{_2}$};
		\node (r3) at (2.85355, 0.51247) [] {\faAndroid$_{_3}$};
		\draw[thick](a1) -- (b2);
		\draw[thick](a2) -- (b1);
		\draw[thick](a2) -- (c2);
		\draw[thick](a2) -- (c3);
		\draw[thick](b2) -- (c1);
		\draw[thick](c1) -- (a2);
		\draw[thick](b2) -- (a3);	  
		\draw[thick](b3) -- (c1);
		\begin{pgfonlayer}{background}
			\node[fit=(a1)(a2)(a3),rounded corners,fill=violet!15,inner xsep=3pt,
	inner ysep=4pt] {};
			\node[fit=(b1)(b2)(b3),rounded corners,fill=green!18,inner xsep=-4pt,
	inner ysep=5pt,
			rotate=28] {};
			\node[fit=(c1)(c2)(c3),rounded corners,fill=cyan!18,inner xsep=-4pt,
	inner ysep=4pt,
			rotate=152] {};
		\end{pgfonlayer}
		\end{tikzpicture}
		\captionsetup{justification=centering}
		\caption{Example $\Gcal$}
		\label{fig:gex}
	\end{subfigure}
	\begin{subfigure}[t]{0.29\textwidth}
		\centering
		\begin{tikzpicture}[scale=1]
		\tikzstyle{vertex}=[circle,fill,scale=0.4,draw]
		\tikzstyle{special vertex}=[circle,fill=red!50,scale=0.3]
		\tikzstyle{square vertex}=[rectangle,fill,scale=0.5,draw]
		\tikzstyle{diamond vertex}=[regular polygon,regular polygon
		sides=3,rotate=45,fill,scale=0.3,draw]
		\node[vertex] at (1,2.598076) (a1) {};
		\node[special vertex] at (1.5,2.598076) (a2) {\faWifi};
		\node[vertex] at (2,2.598076) (a3) {};
		\node[vertex] at (0.25,1.299037) (b1) {};
		\node[special vertex] at (0.5,0.866025) (b2) {\faWifi};
		\node[vertex] at (0.75, 0.433013) (b3) {};
		\node[special vertex] at (2.25,0.433013) (c1) {\faWifi};
		\node[vertex] at (2.5,0.866025) (c2) {};
		\node[vertex] at (2.75,1.299037) (c3) {};
		\node (r1) at (1.5,3.098076) [] {\faAndroid$_{_1}$};
		\node (r2) at (0.1464466,0.51247) [] {\faAndroid$_{_2}$};
		\node (r3) at (2.85355, 0.51247) [] {\faAndroid$_{_3}$};
		\draw[thick,red!30](a1) -- (b2);
		\draw[thick,red!30](a2) -- (b1);
		\draw[thick,red!30](a2) -- (c2);
		\draw[thick,red!30](a2) -- (c3);
		\draw[thick,red!30](b2) -- (c1);
		\draw[thick,red!30](c1) -- (a2);
		\draw[thick,red!30](b2) -- (a3);	  
		\draw[thick,red!30](b3) -- (c1);
		\begin{pgfonlayer}{background}
			\node[fit=(a1)(a2)(a3),rounded corners,fill=violet!15,inner xsep=3pt,
	inner ysep=4pt] {};
			\node[fit=(b1)(b2)(b3),rounded corners,fill=green!18,inner xsep=-4pt,
	inner ysep=5pt,
			rotate=28] {};
			\node[fit=(c1)(c2)(c3),rounded corners,fill=cyan!18,inner xsep=-4pt,
	inner ysep=4pt,
			rotate=152] {};
		\end{pgfonlayer}
		\end{tikzpicture}
		\captionsetup{justification=centering}
		\caption{Unconstrained}
		\label{fig:policy}
	\end{subfigure}
	\begin{subfigure}[t]{0.36\textwidth}
		\centering
		\begin{tikzpicture}[scale=1]
			\tikzstyle{vertex}=[circle,fill,scale=0.4,draw]
			\tikzstyle{special vertex}=[circle,fill=red!50,scale=0.3]
			\tikzstyle{square vertex}=[rectangle,fill,scale=0.5,draw]
			\tikzstyle{diamond vertex}=[regular polygon,regular polygon
			sides=3,rotate=45,fill,scale=0.3,draw]
			\node[vertex] at (1,2.598076) (a1) {};
			\node[special vertex] at (1.5,2.598076) (a2) {\faWifi};
			\node[vertex] at (2,2.598076) (a3) {};
			\node[vertex] at (0.25,1.299037) (b1) {};
			\node[special vertex] at (0.5,0.866025) (b2) {\faWifi};
			\node[vertex] at (0.75, 0.433013) (b3) {};
			\node[vertex] at (2.25,0.433013) (c1) {};
			\node[vertex] at (2.5,0.866025) (c2) {};
			\node[vertex] at (2.75,1.299037) (c3) {};
			\node (r1) at (1.5,3.098076) [] {\faAndroid$_{_1}$};
			\node (r2) at (0.1464466,0.51247) [] {\faAndroid$_{_2}$};
			\node (r3) at (2.85355, 0.51247) [] {\faAndroid$_{_3}$};
			\draw[thick,blue!80](a1) -- (b2);
			\draw[thick,red!30,dashed](a2) -- (b1);
			\draw[thick,blue!80](a2) -- (c2);
			\draw[thick,red!30,dashed](a2) -- (c3);
			\draw[thick,blue!80](b2) -- (c1);
			\draw[thick,red!30,dashed](c1) -- (a2);
			\draw[thick,red!30,dashed](b2) -- (a3);	  
			\draw[thick,black!40,dashed](b3) -- (c1);
			
			\begin{pgfonlayer}{background}
					\node[fit=(a1)(a2)(a3),rounded corners,fill=violet!15,inner xsep=3pt,
			inner ysep=4pt] {};
					\node[fit=(b1)(b2)(b3),rounded corners,fill=green!18,inner xsep=-4pt,
			inner ysep=5pt,
					rotate=28] {};
					\node[fit=(c1)(c2)(c3),rounded corners,fill=cyan!18,inner xsep=-4pt,
			inner ysep=4pt,
					rotate=152] {};
				\end{pgfonlayer}
		\end{tikzpicture}
		\captionsetup{justification=centering}
		\caption{$(\kcomm,\kcomp) = (2,3)$}
		\label{fig:codesign}
	\end{subfigure}
	\caption[Distributed loop closure detection over an exchange graph]
	{\textbf{(a)} An example exchange graph
	$\Gcal$ in a $3$-rendezvous where each robot owns three observations
	(vertices). 
	Each potential inter-robot loop closure can be verified, if at least one
	robot shares its observation with the other robot.
	\textbf{(b)} In the absence of any resource constraint, the robots can collectively
	verify all potential loop closures. The optimal lossless exchange policy
	\citep{Giamou18_ICRA} corresponds to sharing vertices of a
	minimum vertex cover, in this case
	the $3$ vertices marked in red.  \textbf{(c)} Now if the robots are only
	permitted to exchange at most $2$ vertices ($\kcomm = 2$) and verify at most
	$3$ edges ($\kcomp = 3$), they must decide which subset of observations to
	share (marked in red), and which subset of potential loop closures to verify
	(marked in blue). Note that these two subproblems are tightly coupled, as
	the selected edges must be covered by the selected vertices.
  	}
	\label{fig:diagram}
\end{figure*}
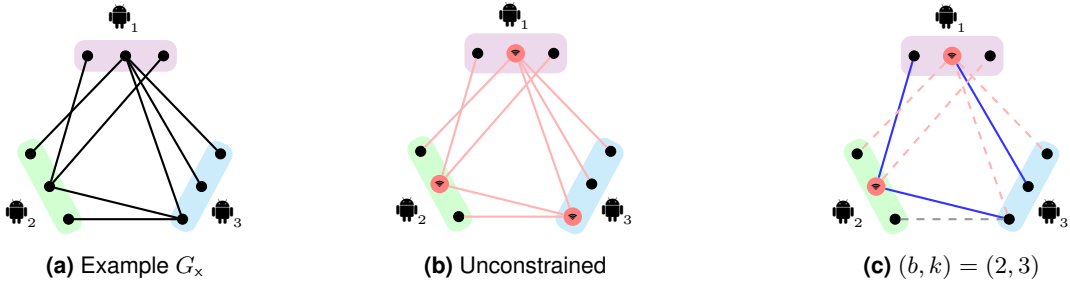

\subsection{Resource Constraints}
\label{sec:budgets}
Under the mild assumption that the computational cost of geometric verification
is uniform across edges, we model the total computational cost of verifying a
subset of edges $\Ecal \subseteq \Eall$ by its cardinality $|\Ecal|$. 
Therefore, imposing a computational budget in this model is equivalent to enforcing the
cardinality constraint $|\Ecal| \leq k$ on the subset of potential matches
selected for verification for some budget $k$. In addition, note that by limiting the
number of verifications one also bounds the number of new edges added to the CSLAM
pose graph, which helps to control the computational cost of CSLAM back-end.

As mentioned earlier, verifying potential matches also incurs a communication
cost: before two robots can verify a potential loop closure, at least one of them must \emph{share} its observation with the
other robot. In graph terms, verifying $\Ecal \subseteq \Eall$ requires robots
to broadcast a subset of their vertices (observations) $\Vcal
\subseteq \Vall$ that \emph{covers} $\Ecal$.\footnote{We say $\Vcal$
  ``covers'' $\Ecal$ iff $\Vcal$ contains at least one of the two vertices
  incident to any edge in $\Ecal$. Also, ``selecting'' a vertex $v \in
  \Vall$ hereafter is synonymous with broadcasting the corresponding (full)
  observation (see Figure~\ref{fig:policy}).} In other words, one can verify
  $\Ecal \subseteq \Eall$ only if there exists a $\Vcal \subseteq \Vall$ such that
  (i) $\Vcal$ covers $\Ecal$, and (ii) broadcasting $\Vcal$ does not
  exceed the allocated communication budget \citep{Giamou18_ICRA,tian18}.
Specifically, we consider three types of communication budgets
(Table~\ref{tab:gcomm}).
First, in Total-Uniform (\TotalUni{}) robots are allowed to exchange at most $b$
observations. This is justified under the assumption of uniform vertex weight
(i.e. observation size) $w$. This assumption is relaxed in
Total-Nonuniform (\TotalNon{}) where total data transmission must be at most
$b$. Finally, in Individual-Uniform (\IndivUni{}), we assume $\Vall$ is
partitioned into $n_b$ blocks and robots are allowed to broadcast at most $b_i$
observations from the $i$th block for all $i \in [n_b]$. A natural partitioning of
$\Vall$ is given by $\VV_1 \uplus \cdots \uplus \VV_r$. In this case, \IndivUni{}
permits robot $i$ to broadcast at most $b_i$ of its observations for all
$i \in [r]$. This model captures the heterogeneous nature of the team.

\begin{table*}[t]
	\setlength{\tabcolsep}{20pt}
	\renewcommand{\arraystretch}{1.3}
	\caption[Three models of communication budgets]
	  {Three models for communication constraints; see
		Section~\ref{sec:budgets} and Problem~\ref{prob:codesign}.}
	\centering
	\begin{tabular}{c||ccc}
		\hline
		\hline
		Type & \textbf{TU}$_b$ & 
		\textbf{TN}$_b$ & \textbf{IU}$_{b_{1:n_b}}$ \\
		\hline\hline
		\multirow{2}{*}{Communication Constraint} & $|\Vcal| \leq b$                & $\sum_{v \in
		\Vcal} w(v) \leq b$   & $|\Vcal \cap \VV_i| \leq b_i$ for
		$i \in [n_b]$\\
		 & Cardinality & Knapsack & Partition Matroid\\
		\hline
		\hline
	\end{tabular}
	\label{tab:gcomm}
\end{table*}

\subsection{Performance Metrics}
\label{sec:objectives}
Given an exchange graph and the abovementioned resource budgets, robots must
decide which budget-feasible subset of potential matches should be verified.
This decision is driven by a
so-called collective performance metric $\fe : 2^{\Eall} \to \Rset_{\geq 0}$.
Here, $\fe(\EE)$ quantifies the \emph{expected} utility gained by verifying
the potential loop closures in $\EE \subseteq \Eall$.
We focus on the class of monotone (non-decreasing) submodular
performance metrics.

\begin{definition}
  \label{def:NMS}
  \normalfont
  For a fixed finite ground set $\Wcal$, a set function $f: 2^{\Wcal} \to \Rset$ is normalized, monotone, and submodular (NMS) if it satisfies the following properties:
  \begin{itemize}
	\item[$\diamond$] Normalized: $f(\varnothing) = 0$.
  \item[$\diamond$] Monotone: for any $\Acal \subseteq \Bcal$, $f(\Acal) \leq f(\Bcal)$.
  \item[$\diamond$] Submodular: for any $\Acal \subseteq \Wcal$ and $\Bcal \subseteq \Wcal$,
	\vspace{-0.2cm}
	\begin{equation}
	  f(\Acal) + f(\Bcal) \geq 
	  f(\Acal \cup \Bcal ) + f(\Acal \cap \Bcal).
	\end{equation}
  \end{itemize}
\end{definition}

In what follows, we briefly review three examples of NMS functions, namely (i) the
expected number of loop closures, (ii) (approximate) expected D-optimality criterion, and (iii) expected weighted tree connectivity.
Note that our framework is compatible with \emph{any} NMS
objectives, and is not limited to the options considered below.

\subsubsection*{1) Expected Number of True Loop Closures.}
\label{subsec:NLC} 
Given an exchange graph $\Gall$, the expected number of loop closures (NLC)
obtained after verifying any $\Ecal \subseteq \Eall$ is simply given by
$\sum_{e \in \Ecal} p(e)$ where $p : \Eall \to [0,1]$ is the probability
function associated to $\Gall$. Therefore, if the team seeks to maximize the
number of (true) inter-robot loop closures it discovers, it must maximize the
following performance metric,
\begin{equation}
  f_{\text{NLC}}(\Ecal) \triangleq  \mathbb{E} \big[ \text{\# of loop closures in
  $\Ecal$}\big] = \sum_{e \in \Ecal} p(e).
  \label{eq:NLC} 
\end{equation}
Note that $f_{\text{NLC}}(\varnothing) = 0$ by definition.
$f_{\text{NLC}}$ is NMS; in fact, this function is modular (i.e.,
$-f_{\text{NLC}}$ is also submodular). Given the importance of
$f_\text{NLC}$, in Section~\ref{sec:modular} we present
near-optimal approximation algorithms specifically for (monotone) modular performance
metrics such as $f_{\text{NLC}}$.

\subsubsection*{2) D-optimality Criterion.}
\label{subsec:FIM}
The D-optimality design criterion (D-criterion), defined as the log-determinant of
the Fisher information matrix (FIM), is one of the most widely adopted design
criteria in the theory of optimal experimental design
with well-known geometrical and information-theoretic interpretations; see \citep{Joshi2009,Pukelsheim1993}.
In particular, this design criterion has been used for sensor selection \citep{Joshi2009,shamaiah2010greedy} and
measurement selection in SLAM \citep{khosoussi2019reliable,carlone2017attention,kasra16wafr}. 

Let $\HH_\text{init} \succ \zero$ denote the information matrix of the joint CSLAM problem
before incorporating the potential loop closures. Moreover, let $\HH_e = \JJ_e^\top
\boldsymbol\Sigma_e^{-1} \JJ_e
\succeq \zero$ be
the information matrix associated to the candidate loop closure $e \in
\Eall$ in which $\JJ_e$ and $\boldsymbol\Sigma_e$ denote the 
measurement Jacobian matrix (evaluated at the current estimate) and the
covariance of Gaussian noise, respectively. Following
\citep{carlone2017attention}, one can \emph{approximate} the
expected gain in the D-criterion as,
\begin{equation}
  f_{\text{FIM}}(\Ecal) \triangleq \log\det \Big(\HH_\text{init} + \sum_{e \in
  \Ecal}
  p(e) \cdot   \HH_e\Big) - \log\det \HH_\text{init}.
  \label{eq:FIM}
\end{equation}
It has been shown that $f_{\text{FIM}}$ is NMS \citep{shamaiah2010greedy,carlone2017attention}.

\subsubsection*{3) Tree Connectivity.}
\label{subsec:WST}
The D-criterion in SLAM can be closely approximated by the weighted number of
spanning trees (WST) (hereafter, tree connectivity) in the graphical
representation of SLAM \citep{khosoussi2019reliable,kasra16icra}. 
\citet{khosoussi2019reliable} use tree connectivity as a topological
surrogate for the D-criterion for selecting (potential) loop closures in planar
pose-graph SLAM. Evaluating tree connectivity is computationally cheaper than evaluating the
D-criterion and, furthermore, does not require any metric knowledge of robots' trajectories.

In the following, we briefly explain how this performance metric can be
evaluated for planar pose-graph CSLAM.
Let $t_{p}(\Ecal)$ and $t_{\theta}(\Ecal)$ denote the weighted number of
spanning trees in a pose graph specified by the edge set
$\Ecal$ whose edges are weighted by the precision of the translational and rotational
measurements, respectively \citep{khosoussi2019reliable}. Furthermore, let
$\Ecal_\text{init}$ denote the set of edges in the CSLAM
pose graph prior to the rendezvous. 
Define 
\begin{align}
  \Phi(\Ecal) \triangleq
  2\, \log \mathbb{E}\,\big[t_{p}(\Ecal_\text{init} \cup \EE)\big] + \log
  \mathbb{E}\,\big[t_{\theta}(\Ecal_\text{init} \cup \EE)\big],
\end{align}
  where
  expectation is taken with respect to the anisotropic random graph model introduced in
  Definition~\ref{def:exchange_graph}---i.e., potential loop closures are
  ``realized'' independently with
  probability assigned by $p : \Eall \to [0,1]$.
  \citet{khosoussi2019reliable} then seek to maximize the following objective.
\begin{align}
  {f}_{\text{WST}}(\Ecal) \triangleq
  \Phi(\Ecal) - \Phi(\varnothing).
  \label{eq:WST}
\end{align}
It is shown in \citep{khosoussi2019reliable,kasra16wafr} that $f_\text{WST}$
is NMS if the underlying pose graph is connected prior to the rendezvous.
\subsubsection*{Final Remarks}
Given an exchange graph, the expected number of loop closures
$f_\text{NLC}$ can be evaluated
efficiently with no additional overhead. Furthermore, this objective is
also suitable for similar distributed matching applications such as
distributed image retrieval or distributed document matching.
While $f_\text{NLC}$ measures performance by the expected number of matches, $f_\text{FIM}$ and $f_\text{WST}$ incentivize verifying
information-rich potential matches that have the highest impact on the CSLAM maximum likelihood
estimate in terms of the determinant of the expected covariance matrix.
Evaluating $f_\text{FIM}$ can be costly (cubic in the
total number of CSLAM poses). Furthermore, it requires the knowledge of the information matrix which
incurs additional communication cost. Alternatively, $f_\text{WST}$ provides a
topological surrogate for $f_\text{FIM}$ that is cheaper to compute and requires
only the knowledge of the topology of the CSLAM pose graph.
Finally, it is worth noting that $f_\text{FIM}$ and $f_\text{WST}$ may be
tempted to pick low-probability but high-impact candidates. To prevent this, one
needs to use a more conservative (higher) threshold on the probability of the
candidates included in the exchange graph (see Section~\ref{sec:exchange}).

\subsection{The Optimization Problem}
\label{sec:problem}
After introducing a number of options for the objective and constraints, we are now
ready to formally define the budgeted exchange-and-verification problem for distributed loop closure detection under
computation and communication constraints. As mentioned earlier, the goal is to
determine (i) which observations must be broadcasted, and (ii) which subset of
potential matches must be geometrically verified, in order to maximize an
NMS performance metric without exceeding the resource budgets.
In graph terms, our goal is to select (i) a
budget-feasible subset $\Vcal$ of all vertices $\Vall$, and (ii) a $k$-subset
$\EE \subseteq \Eall$ of edges covered by $\VV$, such that an NMS $f:
2^{\Eall} \to \Rset_{\geq 0}$ is
maximized.

\begin{problem}
\normalfont
Given an exchange graph $\Gall = (\Vall,\Eall;w,p)$, \textbf{CB} $\in \{\TotalUni{}_{b},\TotalNon{}_{b},\IndivUni{}_{b_{1:n_b}}\}$
and an NMS function $f : 2^{\Eall} \to \Rset_{\geq 0}$, solve the following
optimization problem.
\begin{equation}
	\begin{aligned}
	  & \underset{\Vcal \subseteq \Vall}{\text{maximize}}
	  & & \underset{\substack{\EE \subseteq \edg(\Vcal)\\|\EE| \leq k}}{\max} \hspace{0.2cm} \fe(\EE) \\
	  & \text{subject to}
	  & & \Vcal \text{ satisfies } \textbf{CB}.
	\end{aligned}
\end{equation}
\label{prob:codesign}
\end{problem}

The nested formulation above reflects the inherent structure of the exchange-and-verification problem:
the team must jointly decide which observations to share (outer problem),
and which potential loop closures to verify among the set of verifiable
potential loop closures given the shared observations (inner problem).
Note that Problem~\ref{prob:codesign} is NP-hard in general.
Specifically, when $b$ or $b_i$'s are sufficiently large
(i.e., unbounded communication), this problem coincides with 
general NMS maximization under a cardinality
constraint. No polynomial-time approximation algorithm can provide a constant
factor approximation for this problem better than $1-1/e$, unless P$=$NP; see \citep{krauseSurvey}
for a survey.
This immediately implies that $1-1/e$ is also the approximation barrier for the
general case of Problem~\ref{prob:codesign}.
In the next two sections,
we present approximation algorithms with provable performance guarantees for
variants of this problem. We also consider a similar setting where robots are
subject to individual computational budgets in
Section~\ref{sec:individualBudgets}.

\section{Algorithms for Modular Performance Metrics}
\label{sec:modular}

In this section, we consider a special case of Problem~\ref{prob:codesign} where
$\fe$ is normalized, monotone, and \emph{modular}; i.e., 
$\fe(\varnothing) = 0$ and $\fe(\EE) = \sum_{e \in \EE} \fe(e)$ for all
non-empty $\EE \subseteq \Eall$ where $\fe(e) \geq 0$ for all $e \in \Eall$. 
Problem~\ref{prob:codesign} with modular objectives generalizes the well-known maximum coverage problem on graphs, 
and thus remains NP-hard in general.
Without loss of generality, we use the modular performance metric $f_\text{NLC}$
defined in \eqref{eq:NLC} as a running example in this
section.
In what follows, we present
an efficient constant-factor approximation scheme for Problem~\ref{prob:codesign} with modular objectives under
the communication cost regimes listed in Table~\ref{tab:gcomm}.

Define $\fm : 2^{\Vall} \to \mathbb{R}_{\geq 0}$ such that $\fm(\VV)$ gives the optimal value of the
inner maximization (over edges) in Problem~\ref{prob:codesign} for a given subset of vertices $\VV \subseteq
\Vall$.
For example, for $f_\text{NLC}$ we have,
\begin{equation}
  \fm(\Vcal) \triangleq \max_{\substack{\Ecal \subseteq \edg(\VV) \\ |\Ecal|
  \leq k}} \,\, \sum_{e \in \Ecal} p(e),
\label{eq:fmdef}
\end{equation}
for any $\Vcal \subseteq \Vall$; i.e.,
$\fm(\VV)$ gives the maximum expected number of true inter-robot loop
closures discovered by broadcasting the observations associated to $\VV$ and
verifying at most $k$ potential inter-robot loop closures. 
Note that $\fm(\varnothing) = 0$ by definition.
It is easy to see that the inner maximization problem in
Problem~\ref{prob:codesign} (with monotone modular objectives)
admits a trivial solution and hence $\fm$ can be efficiently evaluated for any
$\VV \subseteq \Vall$:
if $\Vcal$ has more than $\kcomp$ incident edges,
return the sum of top $\kcomp$ edge probabilities in $\edg(\Vcal)$;
otherwise, return the sum of all probabilities in $\edg(\Vcal)$.

\begin{theorem}
	\normalfont
	For any normalized, monotone, and modular $\fe$, the corresponding
	$\fm$---as defined in \eqref{eq:fmdef}---is
	NMS.
	\label{th:fvNMS}
\end{theorem}

Theorem~\ref{th:fvNMS} implies that in the case of (monotone) modular
objectives, the outer (vertex selection) part in Problem~\ref{prob:codesign}
is a special instance of monotone submodular maximization subject to a 
cardinality (\TotalUni{}), a knapsack
(\TotalNon{}), or a partition matroid (\IndivUni{}) constraint. 
These problems admit constant-factor approximation
algorithms \citep{krauseSurvey}. The best performance guarantee in all cases is $1-1/e \approx 0.63$
(Table~\ref{tab:modular_apx}); i.e., in the worst case, the expected number of
correct loop closures discovered by such algorithms is no less than 
$63\%$ of that of an optimal solution.
Among these algorithms, variants of the standard
greedy algorithm are particularly well-suited for our application due to their
computational efficiency.
These greedy algorithms enjoy constant-factor approximation guarantees, albeit with a
performance guarantee weaker than $1-1/e$ in the case of \TotalNon{} and
\IndivUni{}; see the first row of Table~\ref{tab:modular_apx} and
\citep{krauseSurvey}. 
Hereafter, we call this family of greedy algorithms adopted to solve
Problem~\ref{prob:codesign} with (monotone) modular objectives \textsc{\small Modular-Greedy}.

\begin{table}[t]
	\setlength{\tabcolsep}{5pt}
	\renewcommand{\arraystretch}{1.5}
	\caption{Approximation ratio for monotone submodular maximization subject to a cardinality, knapsack, 
	and partition matroid 
 	constraint. 
	Here \textsc{Greedy$^*$} includes simple extensions of the natural greedy
	algorithm. These results are due to \citet{nemhauser1978analysis,
	leskovec2007cost, fisher1978analysis,sviridenko2004note,Calinescu2011}; see
	\citep{krauseSurvey} for a survey.}
	\centering
	\begin{tabular}{c||ccc}
		\hline
		\hline
		Alg. & Cardinality & Knapsack & Partition Mat. \\
		\hline\hline
		\textsc{\small Greedy$^*$} & $1-1/e$  &  $1/2 \cdot
			(1-1/e)$  & $1/2$  \\ \hline
			Best  & $1-1/e$  & 
			$1-1/e$   & $1-1/e$ \\
		\hline
		\hline
	\end{tabular}
	\label{tab:modular_apx}
\end{table}

Algorithm~\ref{alg:mgreedy} provides the pseudocode for \textsc{\small Modular-Greedy} under the \TotalUni{} regime. 
At each iteration, 
we simply select (i.e., broadcast) the next remaining
vertex $v^\star$ with the maximum marginal gain over expected number of true
loop closures (line~\ref{alg:mgrd_v}).
The greedy loop is terminated when the algorithm reaches the communication budget, or when there is no remaining vertex. 
After broadcasting $\Vgrd$, we
verify the top $k$ edges in $\edg(\Vgrd)$ that maximize the sum of
probabilities (line~\ref{alg:mgrd_e}).
A na\"{i}ve implementation of \textsc{\small Modular-Greedy} requires $O(\kcomm\cdot |\Vall|)$ evaluations of $\fm$,
where each evaluation takes $O(|\edg(\VV)| \times \log k)$ time.
The number of evaluations can be significantly reduced by using the
so-called lazy greedy method; see \citep{minoux1978accelerated, krauseSurvey}.

Under \TotalUni{}, \textsc{\small Modular-Greedy} (Algorithm~\ref{alg:mgreedy}) provides the optimal performance guarantee of 
$1-1/e$ \citep{nemhauser1978analysis}. Under \TotalNon{}, the
same greedy algorithm, together with one of its variants that normalizes marginal
gains by vertex weights, provide a performance guarantee of $1/2 \cdot (1-1/e)$; 
see \citep{leskovec2007cost}.
Finally, in the case of \IndivUni{}, selecting the next \emph{feasible} vertex
according to the standard greedy algorithm leads to a performance guarantee of $1/2$ \citep{fisher1978analysis}. 

\begin{algorithm}[t]
	\caption{\textsc{\small Modular-Greedy} \small (\TotalUni{})}\label{alg:mgreedy}
	\begin{algorithmic}[1]
		\renewcommand{\algorithmicrequire}{\textbf{Input:}}
		\renewcommand{\algorithmicensure}{\textbf{Output:}}
		\Require
		\Statex - Exchange graph $\Gcal = (\Vall, \Eall)$
		\Statex - Communication budget $\kcomm$ and computation budget $\kcomp$ 
		\Statex - Modular $\fe: 2^{\Eall} \to \Rset_{\geq 0}$ and $\fm$ as
		defined in \eqref{eq:fmdef}
		\Ensure
		\Statex - A budget-feasible pair 
		$\Vgrd \subseteq \Vall, \Egrd \subseteq \Eall$. 
		\State {$\Vgrd \leftarrow \varnothing$}
		\State {\small \textcolor{green!50!black}{// While satisfying comm. budget}}
		\While {$|\Vgrd| < \kcomm$ } 
			\State {\small \textcolor{green!50!black}{// Select vertex based on marginal gain in $\fm$}}
   			\State {$v^\star \leftarrow \argmax_{v \in \Vall\setminus\Vgrd}
			      				\fm(\Vgrd \cup \{v\})$} \label{alg:mgrd_v}
   			\State {$\Vgrd \leftarrow \Vgrd \cup \{v^\star\}$}
  		\EndWhile
  		\State {\small \textcolor{green!50!black}{// Select top $k$ edges from edges covered by $\Vgrd$}}
  		\State $\Egrd \leftarrow \argmax_{\EE \subseteq \edg(\Vgrd)}\fe(\EE) \,\,
		\text{ s.t. } \,\, |\EE| \leq k$ \label{alg:mgrd_e}
		\State \Return $\Vgrd, \Egrd$
	\end{algorithmic}
\end{algorithm}

\begin{remark}
\normalfont
\citet{Kulik2009MCP} study the problem of {maximum coverage with packing constraint} (MCP),
which includes Problem~\ref{prob:codesign} with modular $\fe$ under \TotalUni{} as a special case.
Our approach differs from \citep{Kulik2009MCP} in two ways. 
Firstly, the algorithm proposed in \citep{Kulik2009MCP} achieves a performance
guarantee of $1 - 1/e$ for MCP by applying partial enumeration, which is
computationally expensive in practice. This additional complexity is due to a
knapsack constraint on edges (``items'' according to \citep{Kulik2009MCP}) which
is unnecessary in our application. As a result, the 
standard greedy algorithm retains the optimal performance guarantee 
without
any need for partial enumeration.
Secondly, in addition to \TotalUni, we study other models of communication budgets (\TotalNon{} and \IndivUni{}),
which leads to more general classes of constraints (i.e., knapsack and partition
matroid) that MCP does not consider. 
\end{remark}

\begin{remark}
\normalfont
It is worth mentioning that in some special cases, the best performance guarantee goes beyond 
$1-1/e$.
For example,
when $\fe$ is modular, 
the communication cost model is \TotalUni{},
and there is no budget on computation ($k = \infty$),
Problem~\ref{prob:codesign} reduces to the well-studied 
maximum coverage problem over a graph \citep{hochbaum1996approximation}. 
In this case, a simple procedure based on pipage rounding can improve the
approximation factor to $3/4$, see \citep{ageevGraphMaxCover}. Furthermore,
if the graph is bipartite, a specialized algorithm can improve the approximation factor to $8/9$ \citep{caskurluBipartiteMaxCover}.
Nonetheless, 
these approaches do not generalize to the regimes considered in our work (e.g., with a computation budget).
Furthermore, they do not retain the computational efficiency and simplicity offered by the greedy algorithm.
\end{remark}

\subsection{Individual Computational Budgets}
\label{sec:individualBudgets}
Problem~\ref{prob:codesign} imposes a single computational budget on the total
number of potential matches verified by the entire team. While this can help to
control the number of verifications as well as the total cost of CSLAM back-end,
in some scenarios one may wish to impose a specific computational budget on each robot such that robot $i$ is only allowed to verify at
most $k_i$ potential matches (similar to
\TotalUni{} vs.~\IndivUni{}). This allows us to control the division of labor
among heterogeneous robots.
Providing performance
guarantees directly for this more general setting is challenging. Instead, in what
follows we show how
Theorem~\ref{th:fvNMS} and Algorithm~\ref{alg:mgreedy} can be generalized to a similar
setup under pairwise computational budgets. In this alternative regime, robots
$i$ and $j$ together are allowed to verify at most $k_{ij}$ potential
matches. We then describe a simple procedure for obtaining pairwise budgets
$\{k_{ij}\}_{i,j \in [r], j>i}$
from a given set of individual budgets $\{k_i\}_{i \in [r]}$, such that robots
are guaranteed to stay within their individual budgets.

For all $i,j \in [r]$ and $i < j$, let $\EE_{ij} \subseteq \Eall$ denote the set of potential matches between
robots $i$ and $j$. Note that $\{\EE_{ij}\}$ partitions $\EE$ into disjoint
(but possibly empty) subsets.
For the modified Problem~\ref{prob:codesign} with pairwise computational budgets and modular
objective $f_\text{NLC}$, the inner maximization
would become
\begin{equation}
  \begin{aligned}
	g_\text{pair}(\Vcal) \,\, \triangleq \,\, & \underset{\Ecal \subseteq \edg(\VV)}{\text{max}}
	& & \sum_{e \in \Ecal} p(e) \\
	& \text{subject to}
	& & |\EE \cap \EE_{ij} | \leq k_{ij}, \; \forall i,j \in [r], i<j.
  \end{aligned}
  \label{eq:fmdefPairwise}
\end{equation}
Note that this is equivalent to imposing a partition matroid constraint over
$\edg(\VV)$.
It is easy to see that for any given $\VV \subseteq \Vall$, $g_\text{pair}$ in
\eqref{eq:fmdefPairwise} can be efficiently evaluated similar to
\eqref{eq:fmdef}: from each $\EE_{ij} \cap
\edg(\VV)$,
simply pick the $k_{ij}$ edges (if available, otherwise pick all) with largest
probability; and return the sum of probabilities of the selected edges.

\begin{theorem}
	\normalfont
	For any normalized, monotone, and modular $\fe$, the corresponding
	$\fm_\text{pair}$---as defined in \eqref{eq:fmdefPairwise}---is
	NMS.
	\label{th:fvNMSPairwise}
\end{theorem}
This theorem states that, similar to $g$ \eqref{eq:fmdef}, $g_\text{pair}:2^{\Vall} \to
\Rset_{\geq 0}$ that arises from pairwise
computational budgets \eqref{eq:fmdefPairwise} is
also NMS. Therefore, in the case of pairwise computational budgets, a modified version of \textsc{\small Modular-Greedy}
(Algorithm~\ref{alg:mgreedy}) in
which vertices are selected greedily according to $g_\text{pair}$ (instead of
$g$) retains the same performance guarantees provided earlier for \textsc{\small
Modular-Greedy}. 

It only remains to explain how pairwise budgets $\{k_{ij}\}$ must be set such
that robots do not exceed a given set of individual computational budgets
$\{k_i\}$. Note that in the pairwise regime, in the worst case robot $i$ has to
verify all selected edges incident to its vertices; i.e., at most $\sum_{{j>i}}
k_{ij} + \sum_{j<i} k_{ji}$ edges. Consequently, to guarantee 
compatibility with individual budgets, it is sufficient for $\{k_{ij}\}$ to satisfy
the following condition,
\begin{equation}
  \sum_{\substack{j \in [r] \\ j > i}} k_{ij} +
  \sum_{\substack{j \in [r] \\ j < i}} k_{ji}
  \leq k_i, \,\,\, \forall i \in
  [r].
  \label{eq:kij}
\end{equation}
Although any such assignment for $\{k_{ij}\}$ is valid, we also take into
account the expected number of loop closures in $\EE_{ij}$ for setting
$k_{ij}$ by solving the following linear program (LP) and rounding down its
solution.
\begin{equation}
  \begin{aligned}
	& \underset{\{k_{ij}\}}{\text{maximize}}
	& & \sum_{\substack{i,j \in [r]\\ j>i}} c_{ij} \, k_{ij} \\
	& \text{subject to}
	& &
  \sum_{\substack{j \in [r] \\ j > i}} k_{ij} +
  \sum_{\substack{j \in [r] \\ j < i}} k_{ji} \leq k_i, \; \forall i \in
  [r],\\
  & & &  0 \leq k_{ij} \leq |\EE_{ij}|, \; \forall i,j \in [r], j>i,
  \end{aligned}
\end{equation}
where $c_{ij}$ is the expected fraction of true loop closures in $\EE_{ij}$,
i.e., 
\begin{equation}
  c_{ij} \triangleq \frac{\mathbb{E}[\text{\# of loop closures in
  $\EE_{ij}$}]}{|\EE_{ij}|} =
  \frac{1}{|\EE_{ij}|} {\sum_{e \in \EE_{ij}} p(e)}.
\end{equation}
The objective in this LP is designed to favor allocating a greater fraction of $k_i$ to
those $k_{ij}$'s that contain a higher percentage of
actual loop closures.

\section{Algorithms for Submodular Performance Metrics}
\label{sec:submodular}
In this section, we study Problem~\ref{prob:codesign} when $\fe : 2^{\Eall} \to
\Rset_{\geq 0}$ is an arbitrary NMS objective.
To the best of our knowledge, no prior work exists on approximation algorithms
for this problem except \citep{Tian18_WAFR}. The approach we took in Section~\ref{sec:modular} for the
special case of monotone modular objectives cannot be extended to this case:
even evaluating the
``generalized'' $g$ \eqref{eq:fmdef} is NP-hard in general.
It is thus unclear
whether any constant-factor approximation can be attained for an arbitrary NMS objective.

In Sections~\ref{sec:infinite_b} and \ref{sec:infinite_k},
we briefly review two relaxations of Problem~\ref{prob:codesign}
with either unlimited communication or unlimited computation budget.
In both cases, Problem~\ref{prob:codesign} reduces to 
known measurement selection problems for which 
constant-factor ($1-1/e$) performance guarantee can be attained via greedy
algorithms. 
Building on these results,
in Section~\ref{sec:submodular_general} 
we consider the most general case of Problem~\ref{prob:codesign}
under both communication and computation budgets.
Specifically, we provide
an efficient approximation algorithm under the \TotalUni{} regime, with a
performance guarantee that depends on the ratio between the two budgets ($b$ and $k$), as well as the maximum degree
of the exchange graph.

\subsection{Unlimited Communication Budget}
\label{sec:infinite_b}
Let us first relax the communication constraint in Problem~\ref{prob:codesign},
i.e., $b = \infty$ in \TotalUni{} and \TotalNon{}, and $b_i = \infty$ for all $i
\in [n_b]$ in \IndivUni{}.
In this case, any subset of at most $k$ edges
can be trivially verified,
e.g.,
by broadcasting one of the two vertices incident to each selected edge.
In terms of Problem~\ref{prob:codesign},
this means that the decision variable corresponding to vertex selection $\Vcal \subseteq \Vall$
together with all constraints involving $\Vcal$ can be dropped. 
The reduced problem is simply,
\begin{equation}
  \underset{\substack{\Ecal \subseteq \Eall \\ |\Ecal|\leq k}}{\text{maximize }}
\,\, \fe(\Ecal),
\label{eq:pe}
\end{equation}
which is the standard monotone submodular maximization subject to a cardinality constraint \citep{krauseSurvey}.
It is well known that
the standard greedy algorithm that sequentially selects the next best edge based on the marginal gain in $\fe$
attains a $1-1/e$ performance guarantee \citep{nemhauser1978analysis}; see also \citep{kasra16wafr,khosoussi2019reliable,carlone2017attention} for
the applications of this result in the context of measurement selection for SLAM.

\subsection{Unlimited Computation Budget}
\label{sec:infinite_k}
Now let us relax the computational constraint in
Problem~\ref{prob:codesign}.
With a little abuse of notation, let $\fv_\text{com}: 2^{\Vall} \to \mathbb{R}_{\geq 0}$
denote the optimal value of the inner maximization in
Problem~\ref{prob:codesign}---this is similar to \eqref{eq:fmdef} in
Section~\ref{sec:modular}, albeit for an arbitrary NMS objective $f$ as opposed
to a monotone modular one.
As $\fe$ is monotone and $k$ is unlimited, 
it is straightforward to see that the optimal value of the inner problem is attained by
selecting all edges covered by $\Vcal$, i.e.,
\begin{equation}
  \fv_\text{com}(\Vcal) \triangleq \fe\big(\edg(\VV)\big).
\label{eq:fvdef}
\end{equation}
The following theorem about $g_\text{com} : 2^{\Vall} \to \Rset_{\geq 0}$ follows from
\citep[Proposition~2.5]{nemhauser1978analysis} and was independently established
in \citep{tian18}.
\begin{theorem}
  \normalfont
  For any NMS $\fe$, the corresponding $\fv_\text{com}$---as defined in \eqref{eq:fvdef}---is normalized, monotone, and submodular (NMS).
  \label{thm:fe2fv}
\end{theorem}
Theorem~\ref{thm:fe2fv} implies that the outer problem is an instance
of monotone submodular maximization subject to a
cardinality (\TotalUni{}), knapsack (\TotalNon{}), or partition matroid (\IndivUni{})
constraint.
For \TotalUni{}, the standard greedy algorithm that selects vertices based on marginal gain in $\fv$ and verifies all associated edges
achieves a $1-1/e$ performance guarantee. 
The corresponding variants for \TotalNon{} and \IndivUni{} achieve a performance guarantee of 
$1/2 \cdot (1-1/e)$ and $1/2$, respectively; see Table~\ref{tab:modular_apx}.

\subsection{The General Case}
\label{sec:submodular_general}

We are now ready to consider Problem~\ref{prob:codesign} with both budget
constraints under \TotalUni{}.
Note that Sections~\ref{sec:infinite_b} and \ref{sec:infinite_k} provide two
complementary approaches.
The first approach is to \emph{greedily select edges} and continue as long as
constraints are not violated. 
Let us call this algorithm \textsc{\small Edge-Greedy}. 
Note that when $b$ is sufficiently large (e.g., when $b > k$),
Problem~\ref{prob:codesign} reduces to the special case considered in Section~\ref{sec:infinite_b}
for which \textsc{\small Edge-Greedy} achieves a $1-1/e$ performance guarantee.
Algorithm~\ref{alg:egreedy} provides the pseudocode. 
Lines~\ref{alg:egrd_1_start}-\ref{alg:egrd_1_end} 
show the standard greedy loop, where at each iteration 
we select the next remaining edge that produces the largest marginal gain in terms of $\fe$.
After adding an edge, we update
our selected vertices to cover the set of selected edges (line~\ref{alg:egrd_vc}).
An easy way to do this is to include one vertex incident to each selected edge.
In case there is extra computation budget,
we also include a heuristic optimization phase (lines~\ref{alg:egrd_2_start}-\ref{alg:egrd_2_end})
in which the algorithm continues to select from
``communication-free'' edges,
i.e., edges that are covered by the currently selected vertices.
The heuristic phase does not contribute to the theoretical guarantee,
but improves the performance in practice.

The second approach is to 
\emph{greedily select vertices},
i.e., broadcast the corresponding observations and verify all edges incident to them,
until violating the communication or computation budget. 
Let us call this algorithm \textsc{\small Vertex-Greedy}. 
When $k$ is sufficiently large (e.g., when $\kcomm <
\lfloor \kcomp / \Delta \rfloor$ where $\Delta$ is the maximum degree of
$\Gcal$),
\textsc{\small Vertex-Greedy} achieves a $1-1/e$ performance guarantee
as noted in Section~\ref{sec:infinite_k}; see Algorithm~\ref{alg:vgreedy} for pseudocode.
Lines~\ref{alg:vgrd_start}-\ref{alg:vgrd_end} show the greedy loop.
At each iteration, 
we select the next feasible vertex that produces the largest
gain in terms of $\fv_\text{com}$ \eqref{eq:fvdef}.
Once a vertex is selected,
all its incident edges
are included for verification (line~\ref{alg:vgrd_ed}). 

\begin{algorithm}[t]
	\caption{\textsc{\small Edge-Greedy} \small (\TotalUni{})}\label{alg:egreedy}
	\begin{algorithmic}[1]
		\renewcommand{\algorithmicrequire}{\textbf{Input:}}
		\renewcommand{\algorithmicensure}{\textbf{Output:}}
		\Require
		\Statex - Exchange graph $\Gcal = (\Vall, \Eall)$
		\Statex - Communication budget $\kcomm$ and computation budget $\kcomp$ 
		\Statex - $\fe: 2^{\Eall} \to \Rset_{\geq 0}$
		\Ensure
		\Statex - A budget-feasible pair 
		$\Vgrd \subseteq \Vall, \Egrd \subseteq \Eall$. 
		\State {$\Vgrd \leftarrow \varnothing$, 
				$\Egrd \leftarrow \varnothing$}
		\State {\small \textcolor{green!50!black}{// While satisfying comm. and comp. budgets}}
		\While {$|\Vgrd| < \kcomm$ \textbf{and}$|\Egrd| < \kcomp$} \label{alg:egrd_1_start}
		\State {\small \textcolor{green!50!black}{// Select edge based on marginal gain in $\fe$}}
	    \State {$e^\star \leftarrow \argmax_{e \in \Eall\setminus\Egrd}
		      				\fe(\Egrd \cup \{e\})$}
	    \State {$\Egrd \leftarrow \Egrd \cup \{e^\star\}$}
	    \State {\small \textcolor{green!50!black}{// Obtain a vertex cover for $\Egrd$}}
	    \State {\small \textcolor{green!50!black}{// e.g., include one vertex incident to each selected edge}}
	    \State $\Vgrd \leftarrow \textsc{\small VertexCover}(\Egrd)$ \label{alg:egrd_vc}
	    \EndWhile \label{alg:egrd_1_end}
	    \State {\small \textcolor{green!50!black}{// Identify ``comm-free'' edges}}
  		\State {$\Ecal_\text{free} \leftarrow 
  						\edg(\Vgrd) \setminus \Egrd $} \label{alg:egrd_2_start}
  		\State {\small \textcolor{green!50!black}{// Heuristic loop that selects ``comm-free'' edges}}
		\While {$|\Egrd| < \kcomp$}
			\State {\small \textcolor{green!50!black}{// Select edge based on marginal gain in $\fe$}}
			\State {$e^\star \leftarrow \argmax_{e \in \Ecal_\text{free} \setminus \Egrd}
				      				\fe(\Egrd \cup \{e\})$}
			\State {$\Egrd \leftarrow \Egrd \cup \{e^\star\}$}
		\EndWhile \label{alg:egrd_2_end}
		\State \Return $\Vgrd, \Egrd$
	\end{algorithmic}
\end{algorithm}

\begin{algorithm}[t]
	\caption{\textsc{\small Vertex-Greedy} \small (\TotalUni{})}\label{alg:vgreedy}
	\begin{algorithmic}[1]
		\renewcommand{\algorithmicrequire}{\textbf{Input:}}
		\renewcommand{\algorithmicensure}{\textbf{Output:}}
		\Require
		\Statex - Exchange graph $\Gcal = (\Vall, \Eall)$
		\Statex - Communication budget $\kcomm$ and computation budget $\kcomp$
		\Statex - $\fe: 2^{\Eall} \to \Rset_{\geq 0}$ and $\fv_\text{com}$ as defined in
		\eqref{eq:fvdef}
		\Ensure
		\Statex - A budget-feasible pair 
		$\Vgrd \subseteq \Vall, \Egrd \subseteq \Eall$. 
		\State {$\Vgrd \leftarrow \varnothing$, 
		$\Egrd \leftarrow \varnothing$}
		\State {\small \textcolor{green!50!black}{// Greedy loop}}
		\While {\textbf{true}} \label{alg:vgrd_start}
			\State {\small \textcolor{green!50!black}{// Identify remaining feasible vertices}}
			\State {$\Vcal_\text{feas} \leftarrow \{v \notin \Vgrd: |\Vgrd \cup \{v\}| \leq b, |\Egrd \cup \edg(v)| \leq k\}$.} 
			\State {\small \textcolor{green!50!black}{// Terminate if there is no more candidate}}
			\If {$\Vcal_\text{feas} = \varnothing$}
				\State{\textbf{break}}
			\EndIf
			\State {\small \textcolor{green!50!black}{// Select vertex based on
			  marginal gain in $\fv_\text{com}$}}
			\State {$v^\star \leftarrow \argmax_{v \in
			  \Vcal_\text{feas}} \fv_\text{com}(\Vgrd \cup \{v\})$} 
			\State {\small \textcolor{green!50!black}{// Include all edges incident to $\Vgrd$}} 
					\State $\Egrd \leftarrow \edg(\Vgrd)$ \label{alg:vgrd_ed}
		\EndWhile \label{alg:vgrd_end}
		\State \Return $\Vgrd, \Egrd$
	\end{algorithmic}
\end{algorithm}

Now let \textsc{\small Submodular-Greedy} be the algorithm according to which one
runs both \textsc{\small Edge-Greedy} and \textsc{\small Vertex-Greedy} and returns the
best solution among the two. 
Note that a na\"{i}ve implementation of \textsc{\small Submodular-Greedy} 
requires $O(b \cdot |\Vall| + k \cdot |\Eall|)$ evaluations of the objective.
In practice, the cubic complexity of evaluating the D-criterion or tree
connectivity in the total of number of poses can be avoided by leveraging the sparse structure of the global pose graph.
This complexity can be further reduced by reusing Cholesky factors
in each round and utilizing rank-one updates; see \citep{khosoussi2019reliable}.
As mentioned before, the number of evaluations can be reduced significantly by using the lazy
greedy method \citep{minoux1978accelerated,krauseSurvey}.
The following theorem provides a performance guarantee for \textsc{\small Submodular-Greedy}
in terms of $b$, $k$, and the maximum degree $\Delta$.
\begin{theorem}
  \normalfont
  Define $\gamma \triangleq \max\,\{b/k,\lfloor k/\Delta \rfloor/b\}$
  and let 
	$\alpha(b,k,\Delta) \triangleq 1-\exp\big(-\min\,\{1,\gamma\}\big)$ for a
	given instance of Problem~\ref{prob:codesign}.
  \textsc{\small Submodular-Greedy} is an $\alpha(b,k,\Delta)$-approximation algorithm
  for this problem.
  \label{thm:submodular_apx}
\end{theorem}

The above performance guarantee $\alpha(b,k,\Delta)$
reflects the complementary nature of \textsc{\small Edge-Greedy} and \textsc{\small
Vertex-Greedy}.
Intuitively, \textsc{\small Edge-Greedy} (resp., \textsc{\small
Vertex-Greedy}) is expected to perform well when computation (resp., communication)
budget is scarce compared to communication (resp., computation) budget. 
To gain more intuition, let us approximate
$\lfloor k/\Delta \rfloor$ in $\alpha(b,k,\Delta)$ with $k/\Delta$.\footnote{This is a reasonable
approximation when, e.g., $b$ is sufficiently large ($b \geq b_0$) since $k/(\Delta b) - 1/b
< {\lfloor k/\Delta \rfloor}/{b} \leq k/( \Delta b)$ and thus the introduced
error in the exponent will be at most $1/b_0$.}
With this simplification, the performance guarantee can be represented as
$\tilde{\alpha}(\kappa,\Delta)$ where $\kappa \triangleq \kcomm/\kcomp$ is the
budgets ratio.
Figure~\ref{fig:sgrd_apx_kappa} shows $\tilde{\alpha}(\kappa,\Delta)$
as a function of $\kappa$
for different values of $\Delta$.

\begin{figure}[t]
	\centering
	\begin{subfigure}[t]{0.29\textwidth}
		\centering
		\includegraphics[width=\textwidth]{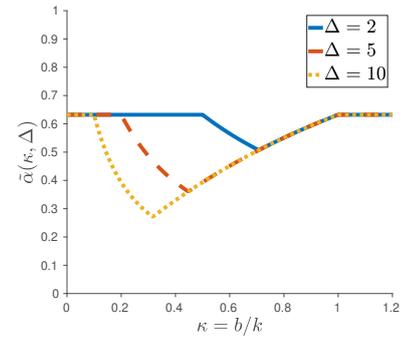}
		\caption{\small $\tilde{\alpha}(\kappa, \Delta)$}
		\label{fig:sgrd_apx_kappa}
	\end{subfigure} \vspace{0.2cm}
	\\
	\begin{subfigure}[t]{0.30\textwidth}
		\centering
		\includegraphics[width=\textwidth]{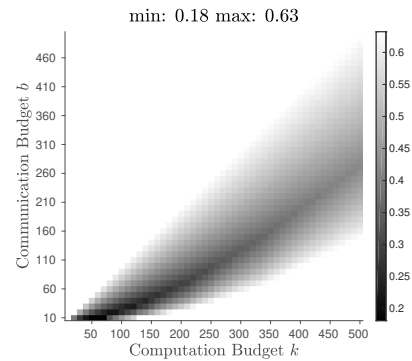}
		\caption{\small KITTI 00 ($\Delta=41$)}
		\label{fig:sgrd_apx_kitti}
	\end{subfigure} \vspace{0.2cm}
	\\
	\begin{subfigure}[t]{0.30\textwidth}
		\centering
		\includegraphics[width=\textwidth]{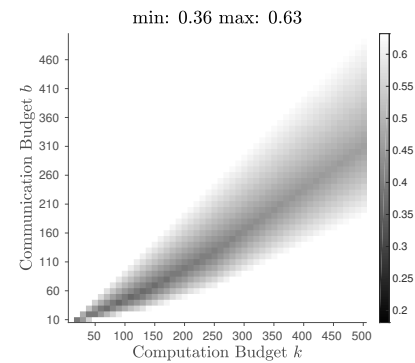}
		\caption{\small KITTI 00 ($\Delta=5$)} 
		\label{fig:sgrd_apx_kitti_maxdeg_5}
	\end{subfigure}
	\caption{\small 
	\textbf{(a)} The approximate \textsc{\small Submodular-Greedy} performance guarantee as a function of $\kappa \triangleq \kcomm / \kcomp$ 
		with different $\Delta$.
	\textbf{(b)} A posteiori performance guarantee in
	an example exchange graph with $\Delta=41$, which varies between $0.18$ and $0.63$.
	\textbf{(c)} A posteriori performance guarantees with the maximum degree capped at $\Delta = 5$.
	The performance guarantee now
	varies between $0.36$ and $0.63$.
  }
	\label{fig:sgrd_apx}
\end{figure} 

We note that for a specific instance of the exchange graph
$\Gcal$, the actual performance guarantee of \textsc{\small Submodular-Greedy} can be higher than $\alpha(b,k,\Delta)$.
This potentially stronger performance guarantee can be
computed \emph{post-hoc}, i.e., after running \textsc{\small Submodular-Greedy} on
the given instance of $\Gcal$---see Lemmas~\ref{lem:egrd_apx} and \ref{lem:vgrd_apx} in Appendix~\ref{pf:submodular_apx}.
As an example, Figure~\ref{fig:sgrd_apx_kitti} shows the 
\emph{post-hoc} performance guarantee on 
an example exchange graph generated from the KITTI~00 dataset (Section~\ref{sec:experiments}).
As we vary the combination of budgets $(\kcomm, \kcomp)$,
the actual performance guarantees vary from 
$0.18$ to $0.63$.
In addition, Theorem~\ref{thm:submodular_apx} 
indicates that reducing $\Delta$ enhances the
performance guarantee $\alpha(b,k,\Delta)$.
This is demonstrated in
Figure~\ref{fig:sgrd_apx_kitti_maxdeg_5}.
After capping $\Delta$ at $5$,
the approximation factors now vary from $0.36$ to $0.63$.
In practice, a major cause of large $\Delta$ is high 
\emph{uncertainty} in the initial set of potential inter-robot loop closures;
e.g., in situations with high perceptual ambiguity, an observation could
potentially be matched to many other observations during the initial phase of
metadata exchange. 
This issue can be mitigated by bounding $\Delta$ or increasing the \emph{fidelity} of metadata.

\begin{remark}
\normalfont
It is worth noting that
$\alpha(b,k,\Delta)$ can be bounded from below by
a function of $\Delta$, i.e., independent of $b$ and $k$. More precisely,
it can be shown that $\alpha(b,k,\Delta) \geq 1 - \exp(-c(\Delta))$ where $ 1/(\Delta+1) \leq c(\Delta) \leq
1/\sqrt{\Delta}$.
This implies that
for a bounded $\Delta$ regime, \textsc{\small Submodular-Greedy}
can be viewed
as a constant-factor
approximation algorithm for Problem~\ref{prob:codesign}.
\end{remark}

\section{Experiments}
\label{sec:experiments}

\begin{figure}[t]
\centering
\begin{subfigure}[t]{0.25\textwidth}
\centering
\includegraphics[width=\textwidth]{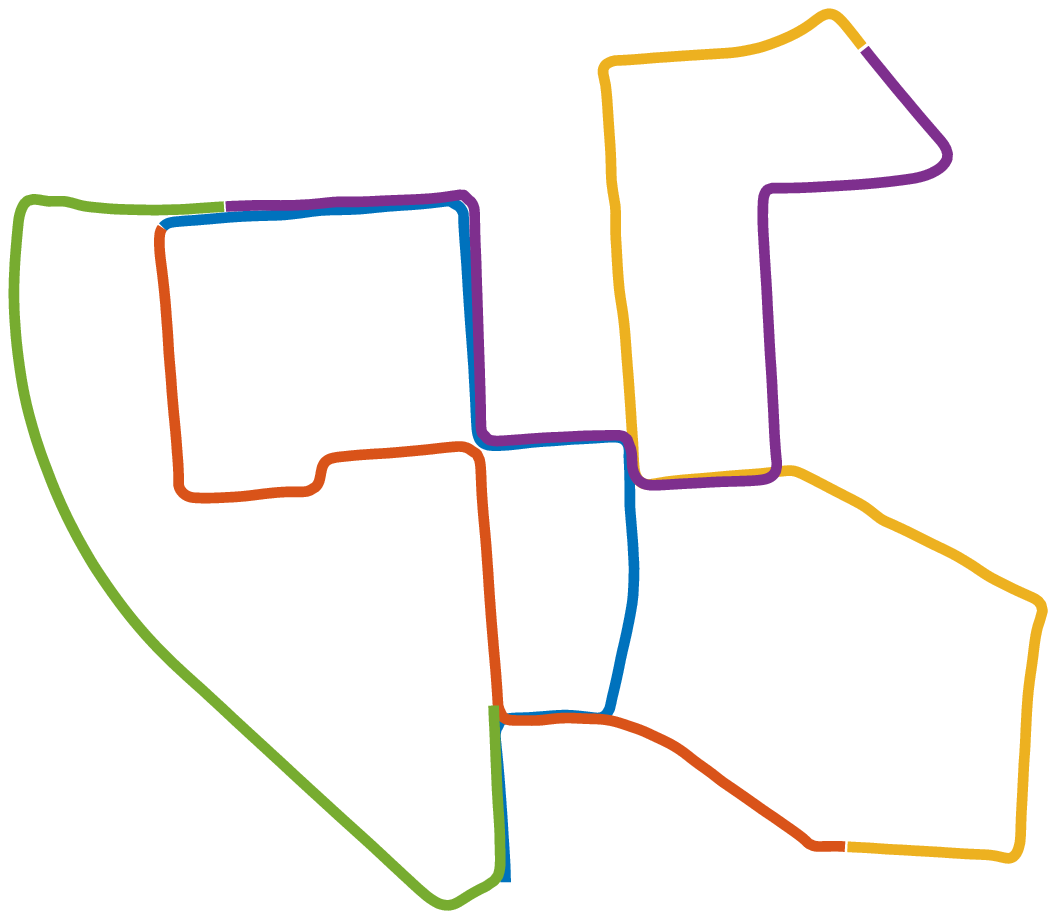}
\caption{\small KITTI 00}
\label{fig:KITTI_00_base_graph}
\end{subfigure}\vspace{0.5cm}
\\
\begin{subfigure}[t]{0.25\textwidth}
\centering
\includegraphics[width=\textwidth]{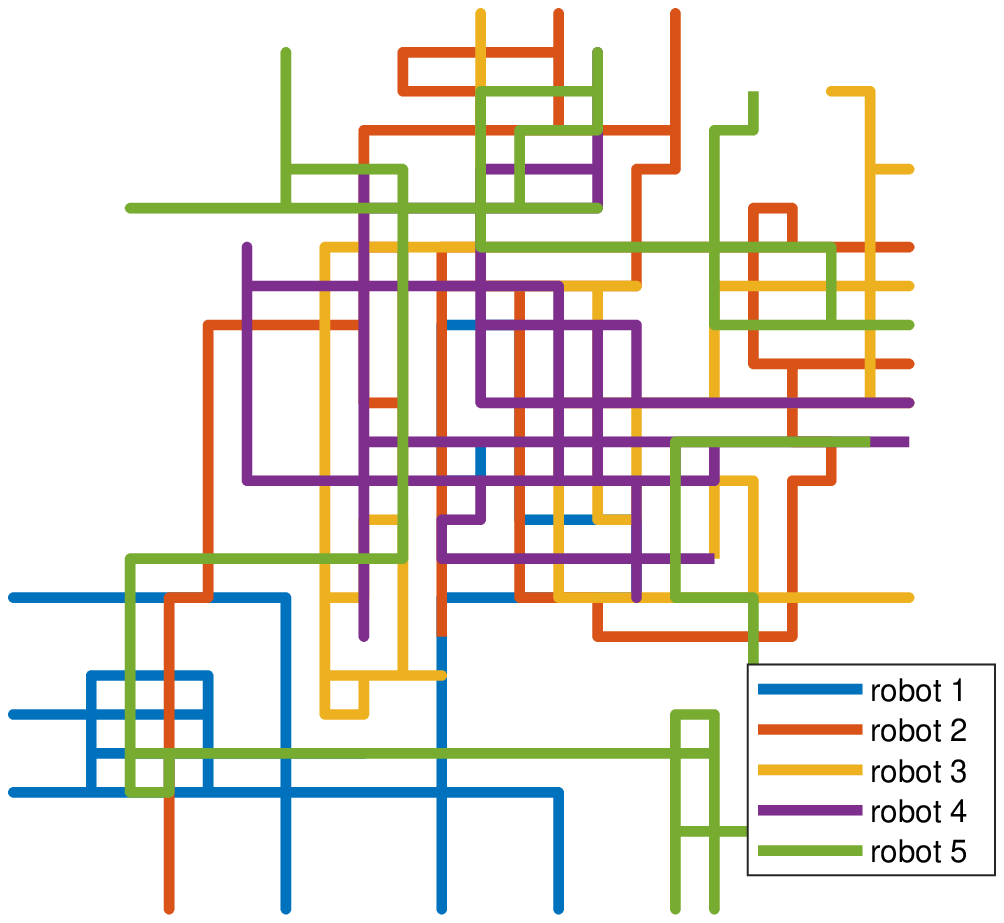}
\caption{\small Simulation}
\label{fig:Atlas_base_graph}
\end{subfigure}
\caption{\small\textbf{(a)} KITTI 00 \textbf{(b)} 2D simulation. Each figure shows trajectories of five robots. Before inter-robot data exchange, trajectories are estimated purely using prior beliefs and odometry measurements, hence the drift displayed in the KITTI trajectories. The 
simulation trajectories shown are the exact ground truth.}
\label{fig:base_graphs}
\end{figure}
We use
KITTI odometry sequence~00 \citep{Geiger2012CVPR} 
and a synthetic Manhattan-like dataset \citep{tian18}
for evaluating
the proposed CSLAM front-end pipeline.
Both datasets are divided into multiple segments to simulate individual robots' trajectories (Figure~\ref{fig:base_graphs}).
For KITTI~00,
visual odometry and geometric verification are performed using a modified version of ORB-SLAM2 \citep{murORB2}.
Trajectories are projected to 2D in order to evaluate
the tree connectivity objective $f_\text{WST}$ \eqref{eq:WST}.
For each image in the sequence,
a NetVLAD vector is extracted as metadata using the pre-trained model provided by \cite{arandjelovic2016netvlad}. 
Probability for each potential loop closure is obtained using
the logistic regression model described in Section~\ref{sec:exchange}.
The model itself is learned offline using KITTI sequence~06 in MATLAB.

In addition to KITTI, we also make use of a synthetic Manhattan-like 
dataset
generated using the 2D 
simulator of g2o \citep{kummerle2011g}.
As opposed to real-world datasets, 
in synthetic datasets we have access to \emph{unbiased}
probability estimates for the potential loop closures.
This is done by simulating each potential loop closure as a random 
Bernoulli variable with a corresponding probability $p$,
where $p$ is drawn randomly from the uniform distribution $\Ucal(0,1)$.
Apart from this, 
synthetic datatsets also have the advantage of being flexible in size
and providing actual ground truths for trajectories and loop closures.
This allows us to evaluate system
performance in terms of the absolute trajectory error (ATE) \citep{SturmATE}.

With the default configurations of ORB-SLAM2, each input image
contains about $2000$ keypoints, where each keypoint (consisting of keypoint coordinate, 3D landmark position, and an ORB descriptor) uses 52 bytes of data.
Thus, a single keyframe translates to about $104$KB of communication payload during geometric verification.
We observed that the variation in the number of keypoints in keyframes is
insignificant, and therefore focus on 
\TotalUni{} and \IndivUni{} communication budget models in our experiments (i.e.,
$w(v) = 1$ for all $v \in \Vall$).

\subsection{Certifying Near-Optimality via Convex Relaxation}
\label{sec:cvx}

To empirically evaluate the performance of the proposed algorithms, one would
ideally need the optimal value (OPT) of Problem~\ref{prob:codesign}. 
However, computing OPT by brute force is impractical even for moderate
problem sizes. 
We therefore compute an upper bound UPT $\geq$ OPT by solving the natural convex
relaxation of Problem~\ref{prob:codesign},
and use UPT as a surrogate for OPT. 
Comparing with UPT provides a post-hoc certificate of near-optimality for solutions returned by the proposed algorithms.
Let $\ppp \triangleq [\pi_1,\dots,\pi_{n}]^\top$ and
$\bell  \triangleq [\ell_1,\dots,\ell_m]^\top$ 
be indicator variables corresponding to vertices (broadcasting an observation)
and edges (verifying a potential loop closure) of the exchange graph $\Gcal$, respectively.
Furthermore, let $\AAA \in \{0,1\}^{n \times m}$ be the undirected incidence matrix of $\Gcal$.
Now Problem~\ref{prob:codesign} can be easily expressed in terms of $\ppp$ and $\boldsymbol{\ell}$.
For example, for modular objectives (Section~\ref{sec:modular}) under the \TotalUni{} communication model, 
Problem~\ref{prob:codesign} is equivalent to solving the following integer linear program
(ILP):

\begin{equation}
	\normalfont
	\begin{aligned}
		& \underset{\ppp,\bell}{\text{maximize}}
		&  \sum_{e \in \Eall} p(e) \, \ell_e \\
		& \text{subject to}
		&	
		\mathbf{1}^{\hspace{-0.05cm}\top}\ppp\leq \kcomm, \\
        &&
		\mathbf{1}^{\hspace{-0.05cm}\top}\bell\leq \kcomp, \\
&& \mathbf{A}^{\hspace{-0.1cm}\top}\ppp \geq \bell,\\
&& \ppp \in \{0,1\}^{n},\\
&& \bell \in \{0,1\}^{m}.
	\end{aligned}
	\label{eq:cvxrelaxation}
\end{equation}
where, with a minor abuse of notation, $\ell_e$ denotes indicator variable corresponding to edge $e$.
The first constraint in \eqref{eq:cvxrelaxation} enforces the communication budget, the second
one enforces the computational budget, and the third one ensures that
selected edges are covered by the selected vertices.
Relaxing the last two integer constraints of the above ILP to $\zero \leq \ppp \leq \mathbf{1}$ and
$\zero \leq \bell \leq \mathbf{1}$ 
gives the natural LP relaxation. The optimal value of the
resulting LP therefore gives an upper bound UPT $\geq$ OPT. 
Similarly, Problem~\ref{prob:codesign} with the D-criterion and tree connectivity objectives (Section~\ref{sec:submodular})
can be expressed as integer determinant maximization problems, 
whose natural convex relaxation gives an upper bound on OPT.
For example, in the case of the D-criterion
objective, Problem~\ref{prob:codesign} can be expressed as:
\begin{equation}
	\normalfont
	\begin{aligned}
		& \underset{\ppp,\bell}{\text{maximize}}
  \,\,\, \log\det \Big(\HH_\text{init} + \sum_{e \in \Eall}
  p(e) \, \ell_e \, \HH_e\Big)\\
  & \text{subject to constraints in \eqref{eq:cvxrelaxation}}.
	\end{aligned}
\end{equation}
Applying the same relaxation results in a convex optimization problem---an
instance of the MAXDET problem subject to additional affine (budgets and
covering) constraints in
\eqref{eq:cvxrelaxation} \citep{vandenberghe1998determinant}.
In our experiments, 
all LPs and ILPs are solved using the built-in solvers in
MATLAB, while MAXDET problems are modeled using the YALMIP toolbox
\citep{Lofberg2004} and solved using SDPT3 \citep{Toh99sdpt3} in MATLAB.

\subsection{Metadata Analysis}
\label{sec:metadata_experiments}
\begin{figure}[t]
\centering
\begin{subfigure}[t]{0.33\textwidth}
\centering
\includegraphics[width=\textwidth]{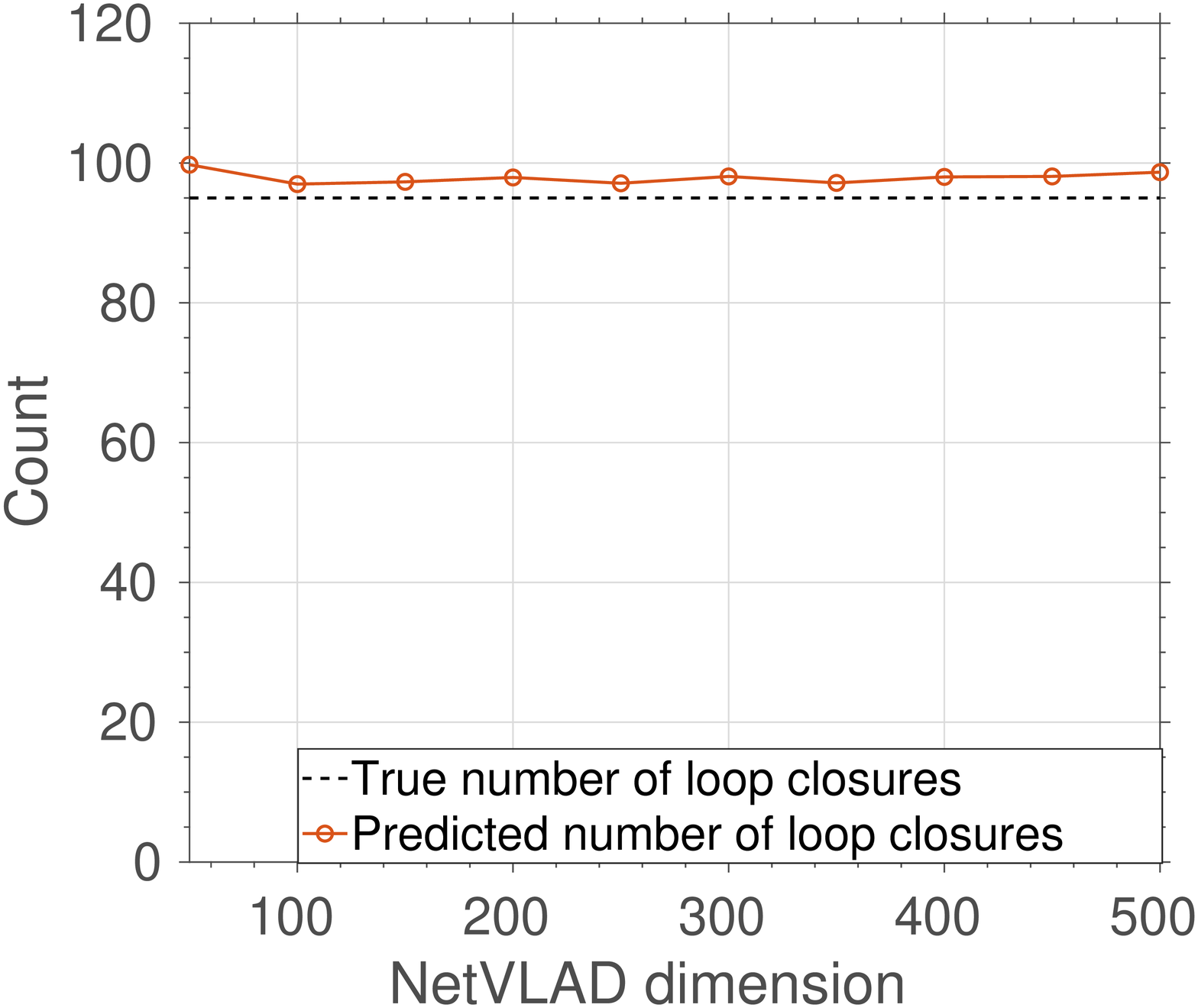}
\caption{Predicted number of loop closures (validation)}
\label{fig:metadata_nlc_val}
\end{subfigure}
\\
\begin{subfigure}[t]{0.33\textwidth}
\centering
\includegraphics[width=\textwidth]{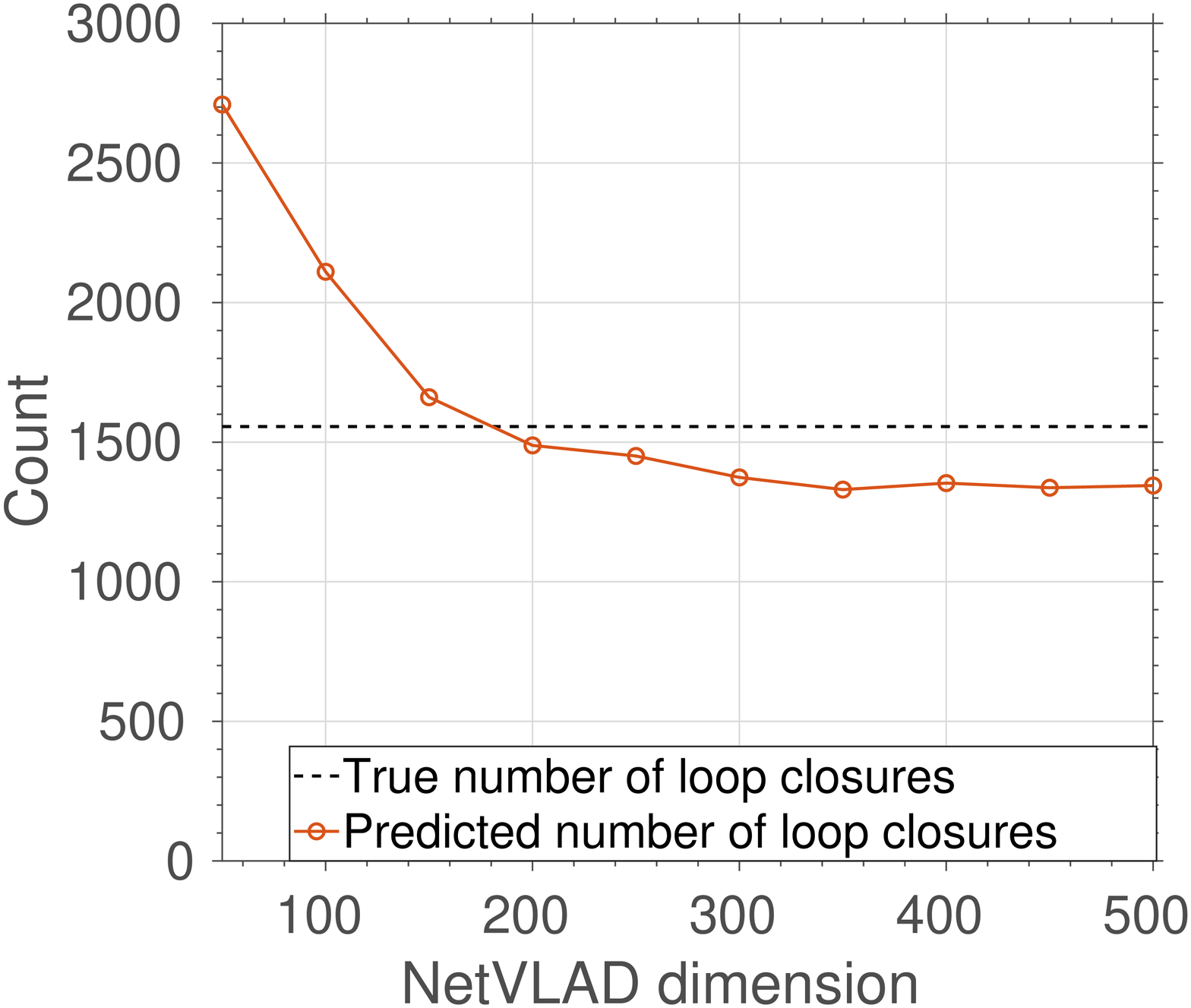}
\caption{Predicted number of loop closures (test)}
\label{fig:metadata_nlc_test}
\end{subfigure}
\caption{\small 
Number of true loop closures predicted by the logistic regression model
in \textbf{(a)} validation and \textbf{(b)} test set, 
as a function of NetVLAD dimension.
Blue line shows the true number of inter-robot loop closures in each dataset. 
}
\label{fig:metadata_nlc}
\end{figure}

\begin{figure}[t]
\centering
\begin{subfigure}[t]{0.33\textwidth}
\centering
\includegraphics[width=\textwidth]{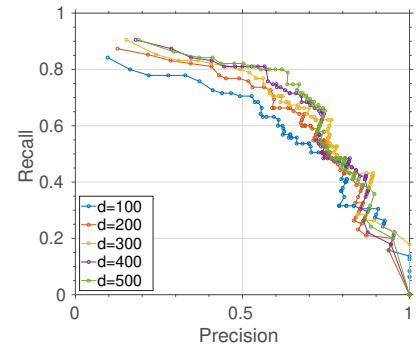}
\caption{Precision-recall (validation)}
\label{fig:metadata_pr_val}
\end{subfigure}
\\
\begin{subfigure}[t]{0.33\textwidth}
\centering
\includegraphics[width=\textwidth]{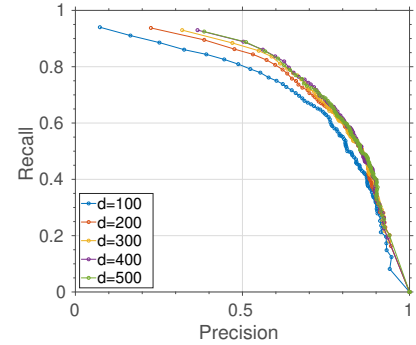}
\caption{Precision-recall (test)}
\label{fig:metadata_pr_test}
\end{subfigure}
\caption{\small
Precision-recall curves in \textbf{(a)} validation and \textbf{(b)} test set.
Each curve is generated using with a fixed NetVLAD dimension and a varying probability threshold
that classifies
if a pair of images is a loop closure.}
\label{fig:metadata_pr}
\end{figure}

In this section, we describe the training and evaluation process for
the logistic regression model proposed in Section~\ref{sec:exchange}. 
For training, 
KITTI~06 is divided into a training set and a validation set,
where the training set is used to learn the parameters of the logistic regression 
classifier, 
and the validation set is used to determine the dimension of NetVLAD
vectors.
The entire KITTI~00 sequence is used as the test set.

The main hyperparameter studied in our experiments is $\ndim$.
A good value for $\ndim$ should strike a balance between 
size and 
quality of the learned probabilities. 
For this, we conduct two sets of quantitative analyses.
First,
we examine the expected number of loop closures in the exchange graph $\sum_{e
\in \Eall} p(e)$ predicted by the learned model 
as $\ndim$ varies from $50$ to $500$. 
In this experiment, no probability threshold is imposed when forming the exchange graph, i.e., $\pthresh = 0$.
Consequently, $\Eall$ contains every inter-robot edges as a potential loop closure.
For each value of $\ndim$, 
the model is re-trained on the training set
and re-evaluated on the validation and test sets.
The results are shown in Figure~\ref{fig:metadata_nlc}.
For reference, we also plot the ground truth number of true loop closures in both 
datasets (black dashed line). 
On the validation set,
the learned models consistently exhibit good performance
with a maximum prediction error of 4.76 across all values of $\ndim$.
Performance on the test set displays a more interesting variation.
For small values of $\ndim$,
the model significantly overestimates the 
number of loop closures. 
This suggests that reducing $\ndim$ too much
leads to degraded accuracy in the predicted probabilities.
As $\ndim$ becomes larger, the predicted value stabilizes
and yields an average error of $219$,
revealing the limit of the learned model in generalizing 
to new datasets. 

Second, 
Figure~\ref{fig:metadata_pr}
shows the precision-recall curves corresponding to
selected values of $\ndim$ on the validation and test set.
Each curve is generated with a fixed $\ndim$
and varying probability threshold
that classifies if a pair of images constitutes an actual loop closure.
Precision is computed as the percentage of true loop closures within 
the set of all predicted loop closures.
Recall is computed as the percentage of discovered loop closures 
within the set of all loop closures in the dataset.
As expected, increasing $\ndim$ improves both precision and recall 
as the metadata is able to capture more information. 
Nevertheless, as $\ndim$ increases,
the marginal improvement in precision and recall decreases.
Based on these results,
we fix $\ndim = 200$ in all subsequent experiments.
With this value, a single NetVLAD vector results in around
$0.8$KB of data payload, 
which is still very lightweight 
compared to the data payload of a full observation ($104$KB).

\subsection{Results with Modular Objectives}
\begin{figure}[t]
\begin{subfigure}[t]{0.23\textwidth}
\centering
\includegraphics[width=\textwidth]{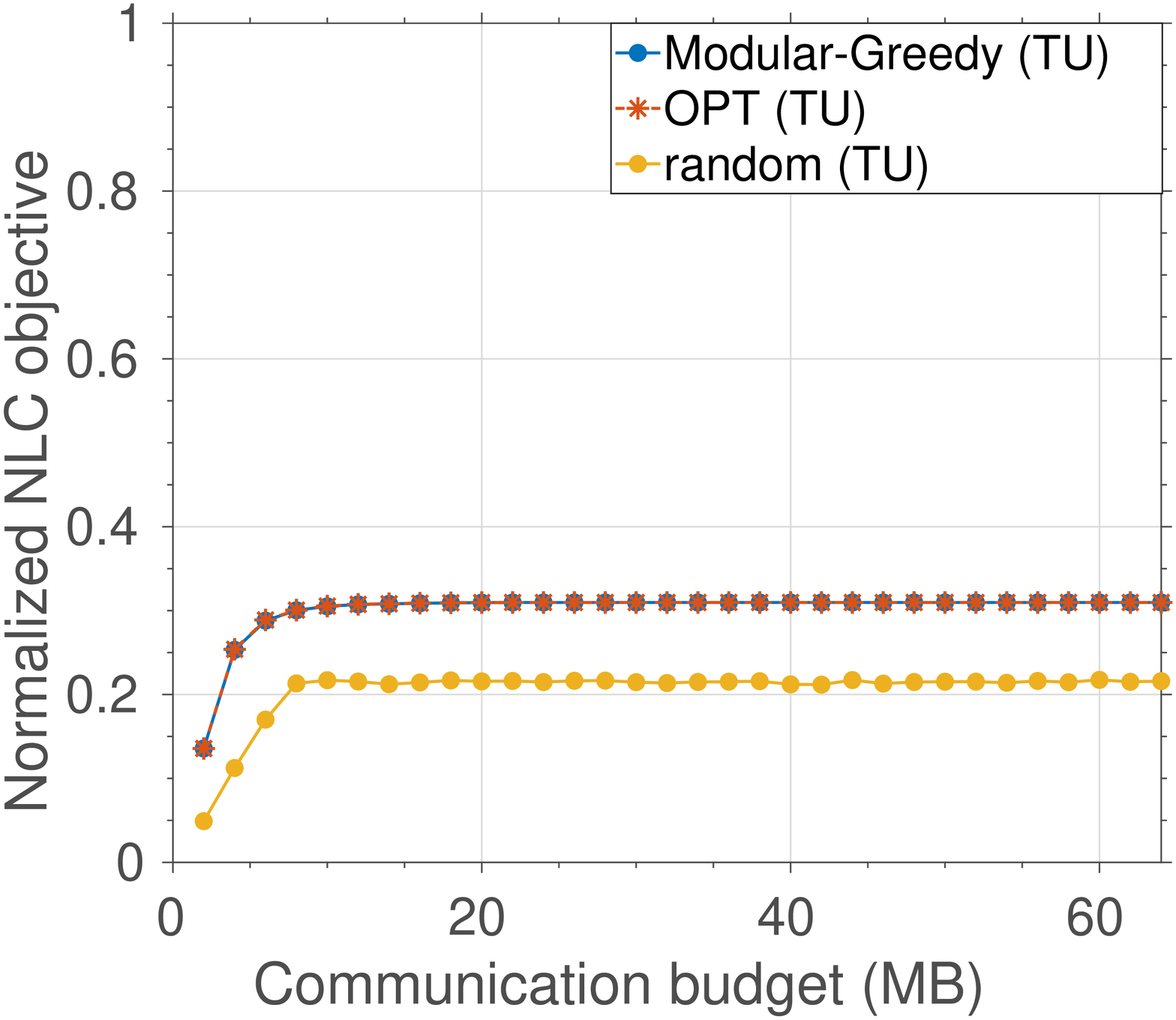}
\caption{KITTI00 $k=300$ (TU)}
\label{fig:KITTI00_mgreedy_k300}
\end{subfigure}
\hfill
\begin{subfigure}[t]{0.23\textwidth}
\centering
\includegraphics[width=\textwidth]{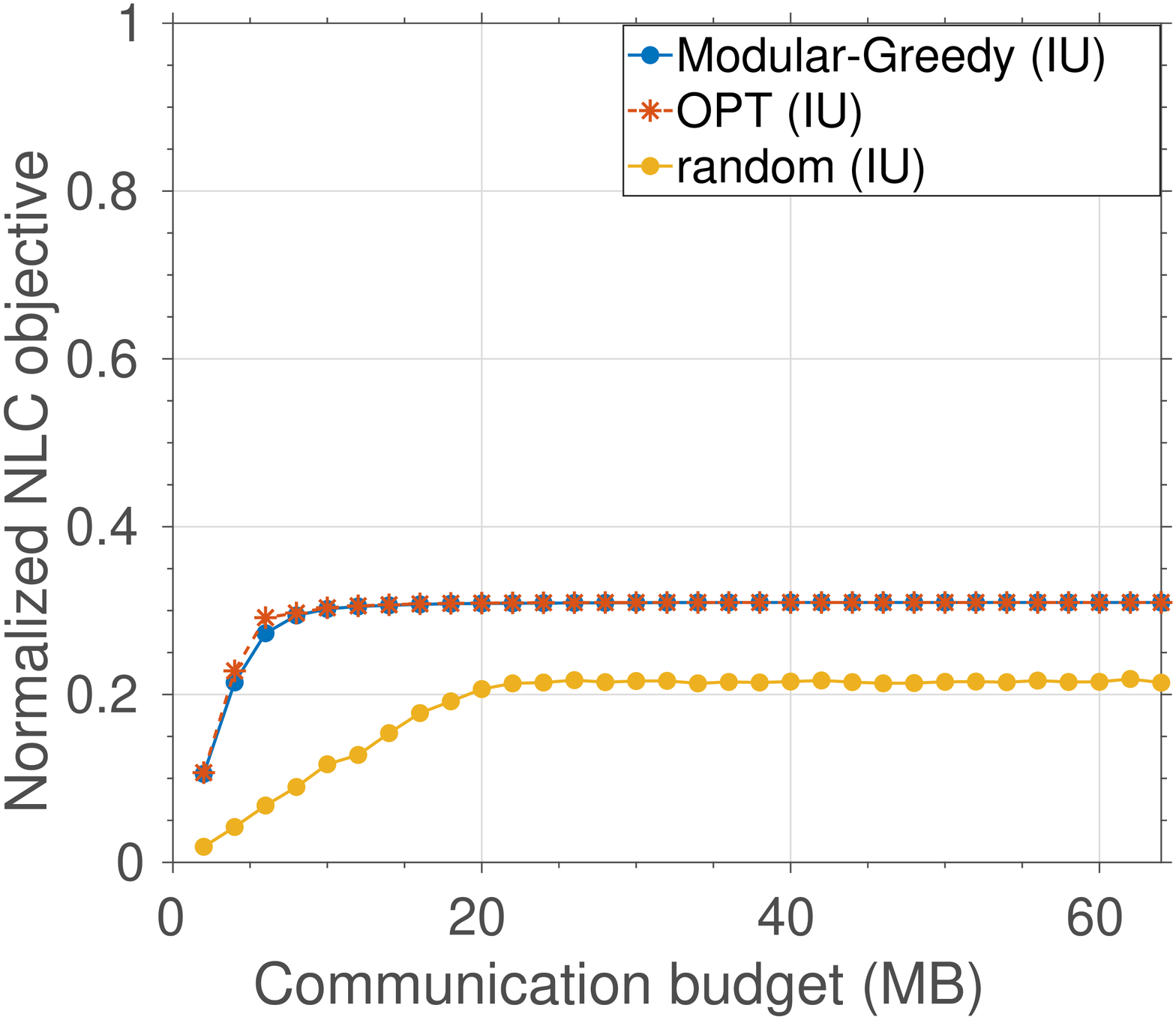}
\caption{KITTI00 $k=300$ (IU)}
\label{fig:KITTI00_mgreedy_IU_k300}
\end{subfigure}
\\
\begin{subfigure}[t]{0.23\textwidth}
\centering
\includegraphics[width=\textwidth]{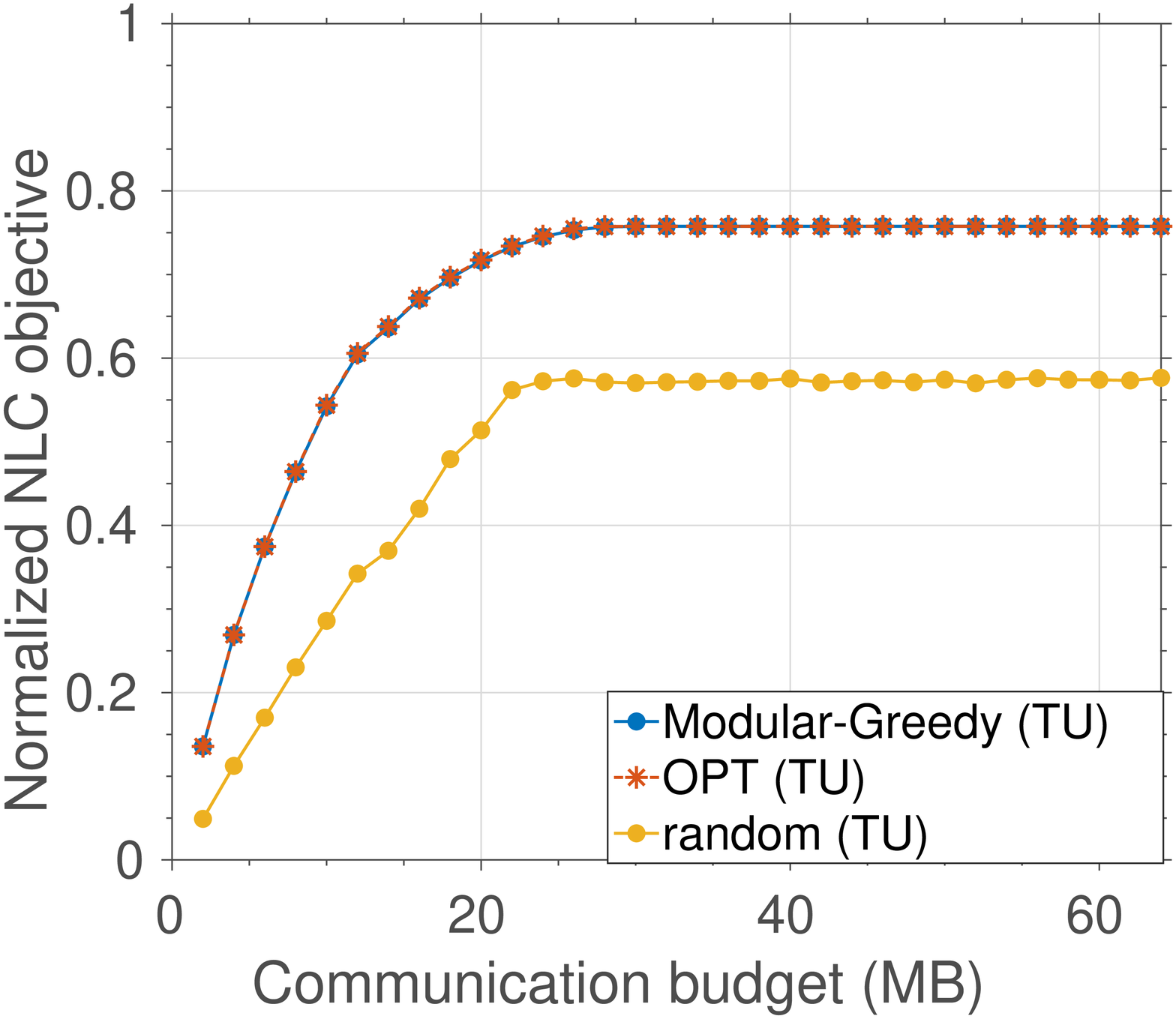}
\caption{KITTI00 $k=800$ (TU)}
\label{fig:KITTI00_mgreedy_k800}
\end{subfigure}
\hfill
\begin{subfigure}[t]{0.23\textwidth}
\centering
\includegraphics[width=\textwidth]{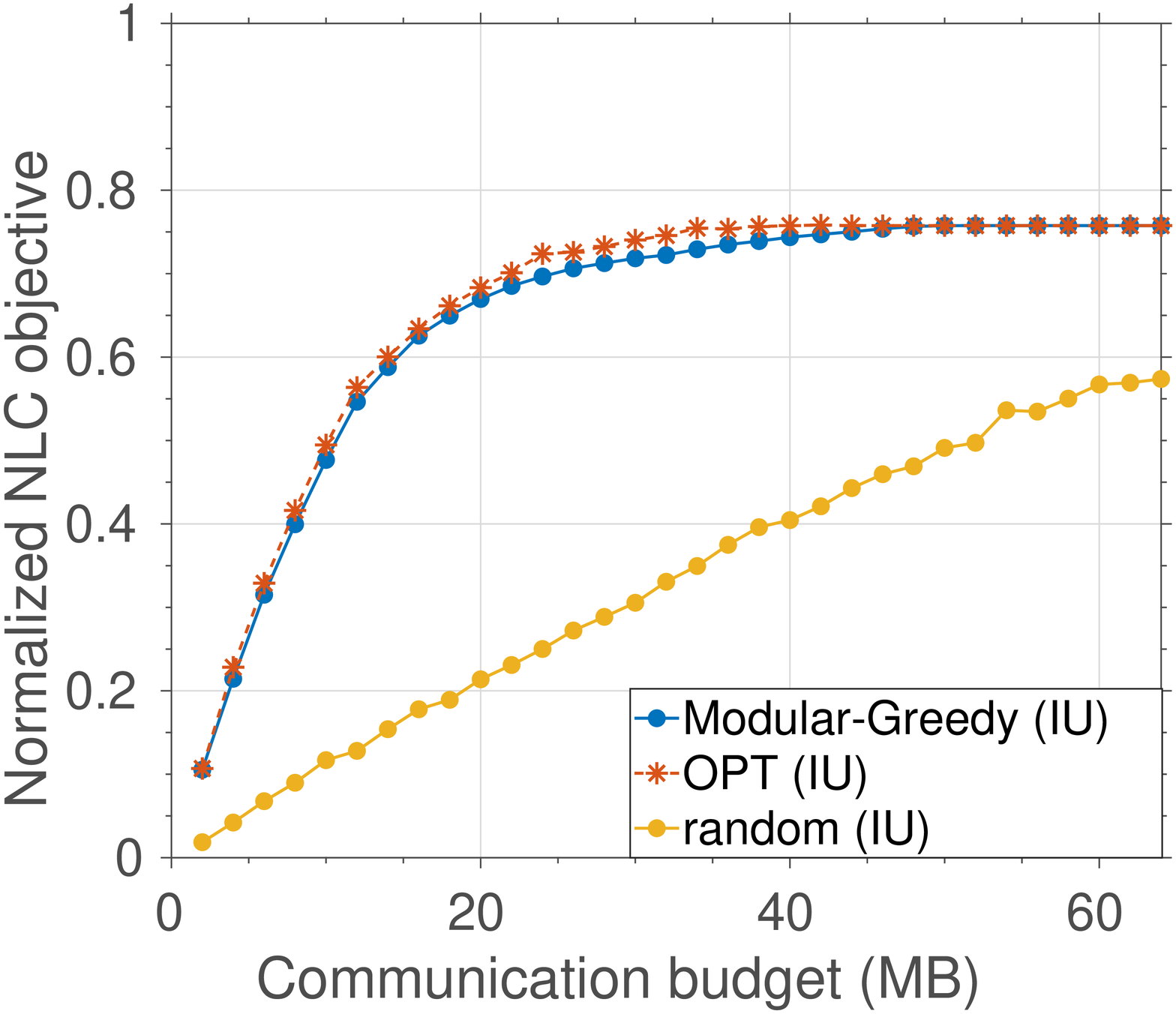}
\caption{KITTI00 $k=800$ (IU)}
\label{fig:KITTI00_mgreedy_IU_k800}
\end{subfigure}
\\
\begin{subfigure}[t]{0.23\textwidth}
\centering
\includegraphics[width=\textwidth]{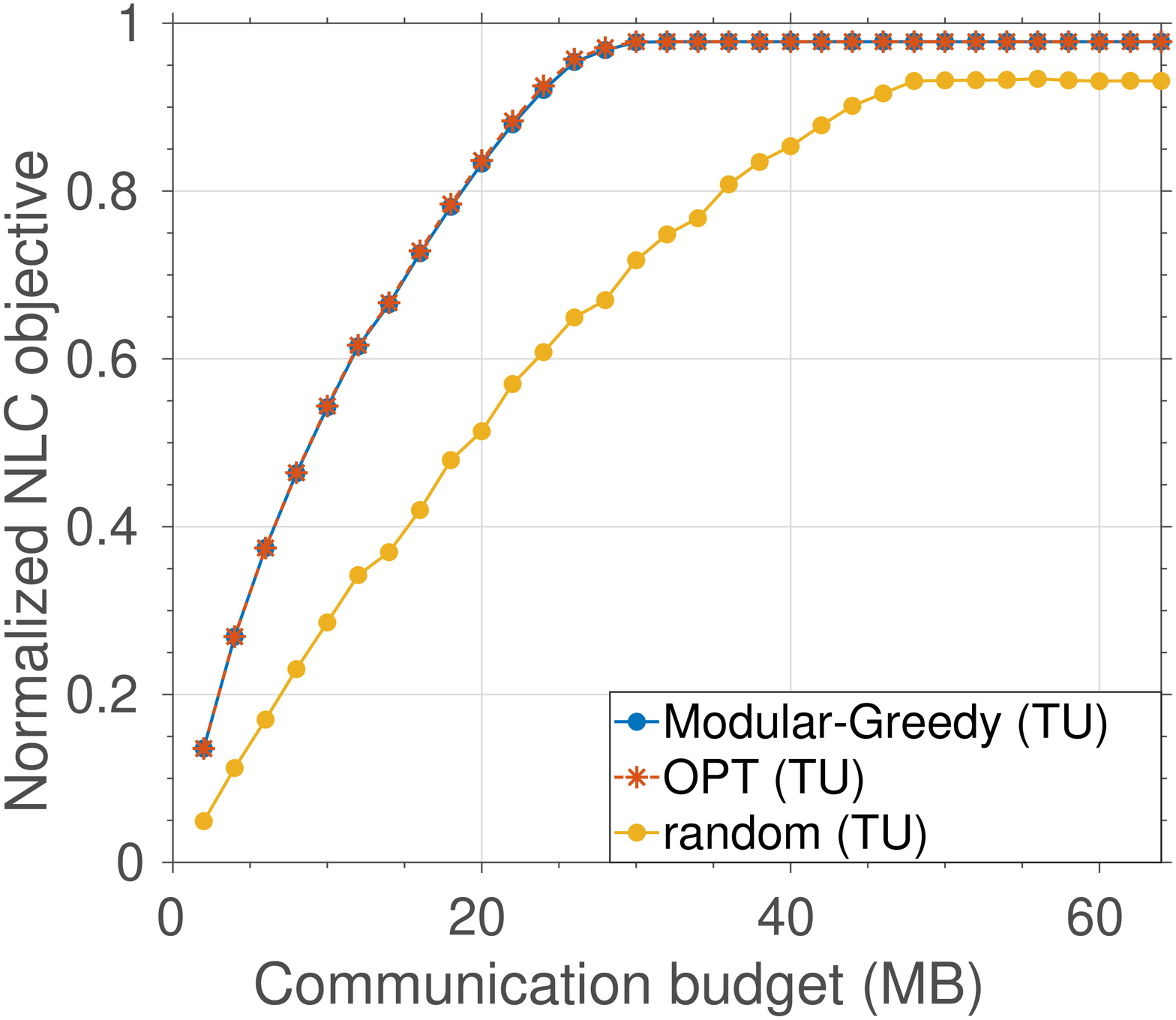}
\caption{KITTI00 $k=1300$ (TU)}
\label{fig:KITTI00_mgreedy_k1300}
\end{subfigure}
\hfill
\begin{subfigure}[t]{0.23\textwidth}
\centering
\includegraphics[width=\textwidth]{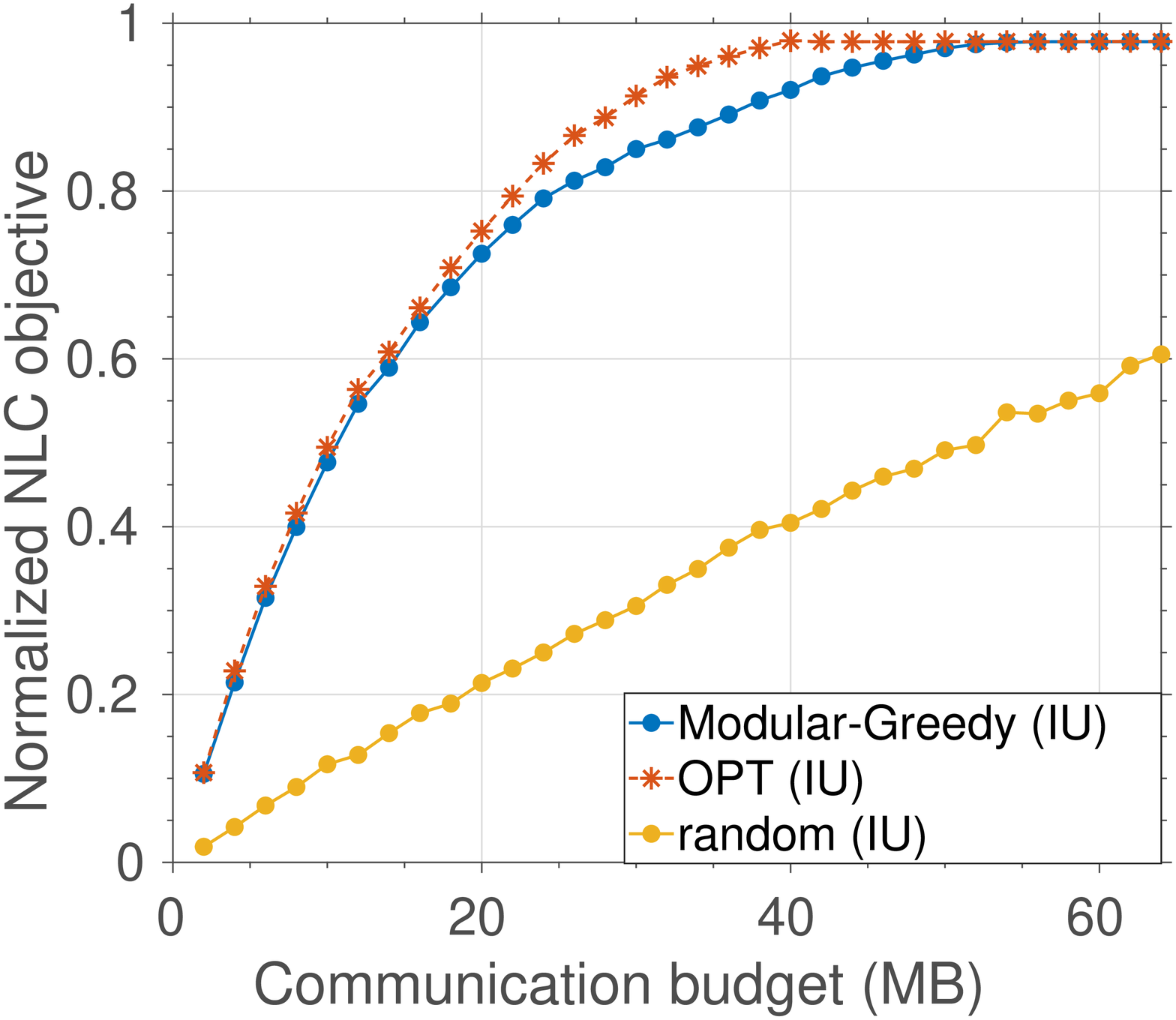}
\caption{KITTI00 $k=1300$ (IU)}
\label{fig:KITTI00_mgreedy_IU_k1300}
\end{subfigure}
\caption{\small Performance of \textsc{\small Modular-Greedy} (Algorithm~\ref{alg:mgreedy})
under the \TotalUni{} (left column) and \IndivUni{} (right column) communication cost models.
Each figure uses a fixed computation budget ($k=300,800,1300$). 
The horizontal axis shows the total communication budget $b$ 
varying from $2$MB to $62$MB.
The vertical axis shows the normalized NLC objective.
For \IndivUni{}, 
each value of $b$ is equally divided among robots to create local communication budgets $b_i$.
The modular greedy solution is compared with the optimal solution from ILP as well as the solution returned by the baseline random algorithm. 
In both \TotalUni{} and \IndivUni{}, \textsc{\small Modular-Greedy} achieves 
near-optimal performance.
\label{fig:KITTI00_mgreedy}
}
\end{figure}

We first evaluate our proposed framework in instances
where the objective function is modular. 
For this, we use the expected number of true loop closures 
$f_\text{NLC}$ \eqref{eq:NLC} as the objective function.
The input exchange graph is generated from the KITTI~00 sequence
with NetVLAD dimension $\ndim = 200$ and
probability threshold $\pthresh = 0.2$.
Under this setting, the exchange graph
contains $1395$ edges, among which $1036$ are real inter-robot loop
closures.
The maximum degree $\Delta$ in this case is $16$.
The na\"{i}ve exchange policy that broadcasts an observation for each potential
loop closure results in $62$MB of total communication.
Note that after dividing KITTI~00 into five robots,
the entire dataset amounts to less than a minute of parallel localization and mapping.
Practical missions with longer durations thus produce significantly larger problem instances.

The first column of Figure~\ref{fig:KITTI00_mgreedy}
shows the performance of the proposed \textsc{\small Modular-Greedy} algorithm (Algorithm~\ref{alg:mgreedy}) under the \TotalUni{}
communication cost model.
In each figure, we use a fixed computation budget $k$
and vary the communication budget $b$ from
$2$MB to $62$MB.
Performance is evaluated in terms of the \emph{normalized} objective, 
i.e., the achieved objective divided by the maximum achievable objective
given infinite budgets. 
In the modular case, we compute the optimal solution (OPT)
by solving the corresponding ILP.
We also compare our results with a baseline algorithm that
randomly selects $b$ vertices and
then selects $k$ edges incident to the selected vertices.
The resulting values 
are averaged across 10 runs to alleviate the effects of random sampling.
Our results clearly confirm the near-optimal performance of \textsc{\small Modular-Greedy}.
In fact, the maximum \emph{unnormalized} optimality gap across all experiments is $4.72$,
i.e., the achieved objective and the optimal objective only differ by $4.72$ 
expected loop closures.

The second column of Figure~\ref{fig:KITTI00_mgreedy} 
shows the performance of \textsc{\small Modular-Greedy}
under the \IndivUni{} communication cost model. 
Similar to the \TotalUni{} experiments,
each figure uses a fixed computation budget $k$ and the same range for 
the total communication budget $b$ ($2$MB to $62$MB).
Each value of $b$ is equally divided among robots to create the local communication budgets
$b_i$'s. 
Performance of \textsc{\small Modular-Greedy} is again compared with the optimal solution (OPT)
by solving the ILP, as well as the baseline random algorithm in the \IndivUni{} regime.
In most cases,
the performance of \textsc{\small Modular-Greedy} is close to optimal,
although the maximum optimality gap is larger 
(see the last figure with $k=1300$)
compared to the \TotalUni{} regime. 
This observation matches our theoretical analysis,
as the established performance guarantee for \textsc{\small Modular-Greedy} in \IndivUni{} (i.e.,
partition matroid) is weaker than in \TotalUni{} (i.e., cardinality constraint); see
Table~\ref{tab:modular_apx}.

\begin{figure}[t]
\centering
\begin{subfigure}[t]{0.30\textwidth}
\centering
\includegraphics[width=\textwidth]{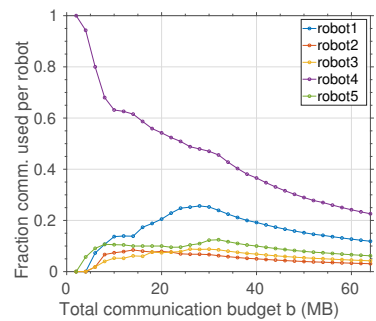}
\caption{\textsc{\small Modular-Greedy} \TotalUni{} ($k=1300$)}
\label{fig:NLC_TU_comm_labor}
\end{subfigure}
\\
\begin{subfigure}[t]{0.30\textwidth}
\centering
\includegraphics[width=\textwidth]{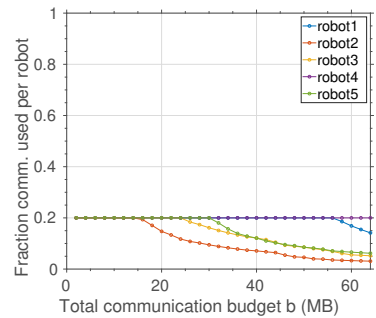}
\caption{\textsc{\small Modular-Greedy} \IndivUni{} ($k=1300$)}
\label{fig:NLC_IU_comm_labor}
\end{subfigure}
\caption{\small 
Fraction of total communication budget used by each robot
in the solutions of \textsc{\small Modular-Greedy}
under \textbf{(a)} \TotalUni{}
and \textbf{(b)} \IndivUni{}.
Both experiments use the same computation budget $k=1300$
and the same range of total communication budget $b$
(2MB-62MB).  
For \IndivUni{}, 
each value of $b$ is equally divided among robots to create local communication budgets $b_i$. 
}
\label{fig:NLC_comm_labor}
\end{figure}

Furthermore, 
the last row (Figure~\ref{fig:KITTI00_mgreedy_k1300} and \ref{fig:KITTI00_mgreedy_IU_k1300}) 
reveals another interesting comparison between \TotalUni{} and \IndivUni{}.
With the same computation budget $k = 1300$, 
\textsc{\small Modular-Greedy} under \TotalUni{} 
reaches maximum performance at $b = 30$MB.
Under \IndivUni{}, however, the algorithm reaches maximal performance much later, with a total communication of $b = 54$MB.
This is due to the fact that inherently,
\IndivUni{} trades off
team performance with 
the balance of induced individual communications.
Such extra considerations
leads to a more restricted search space and as a result, 
a lower objective value given the same total communication budget $b$.

Figure~\ref{fig:NLC_comm_labor}
shows the fraction of total communication budget used by each of the five robots 
in the solutions of \textsc{\small Modular-Greedy}
under \TotalUni{} and \IndivUni{}.
The computation budget is fixed to be $k=1300$ under both regimes.
As \TotalUni{}
does not regulate individual communications,
the resulting distribution of labor could be arbitrarily skewed;
see Figure~\ref{fig:NLC_TU_comm_labor}.
For example, with $b=30$MB, 
robot~4 (purple) broadcasts 50\% of all vertices selected by the team.
In contrast, in \IndivUni{},
the amount of data transmitted by each robot 
is explicitly bounded 
(through local budgets $b_i$'s);
see Figure~\ref{fig:NLC_IU_comm_labor}.

\begin{figure}[t]
\centering
\begin{subfigure}[t]{0.30\textwidth}
\centering
\includegraphics[width=\textwidth]{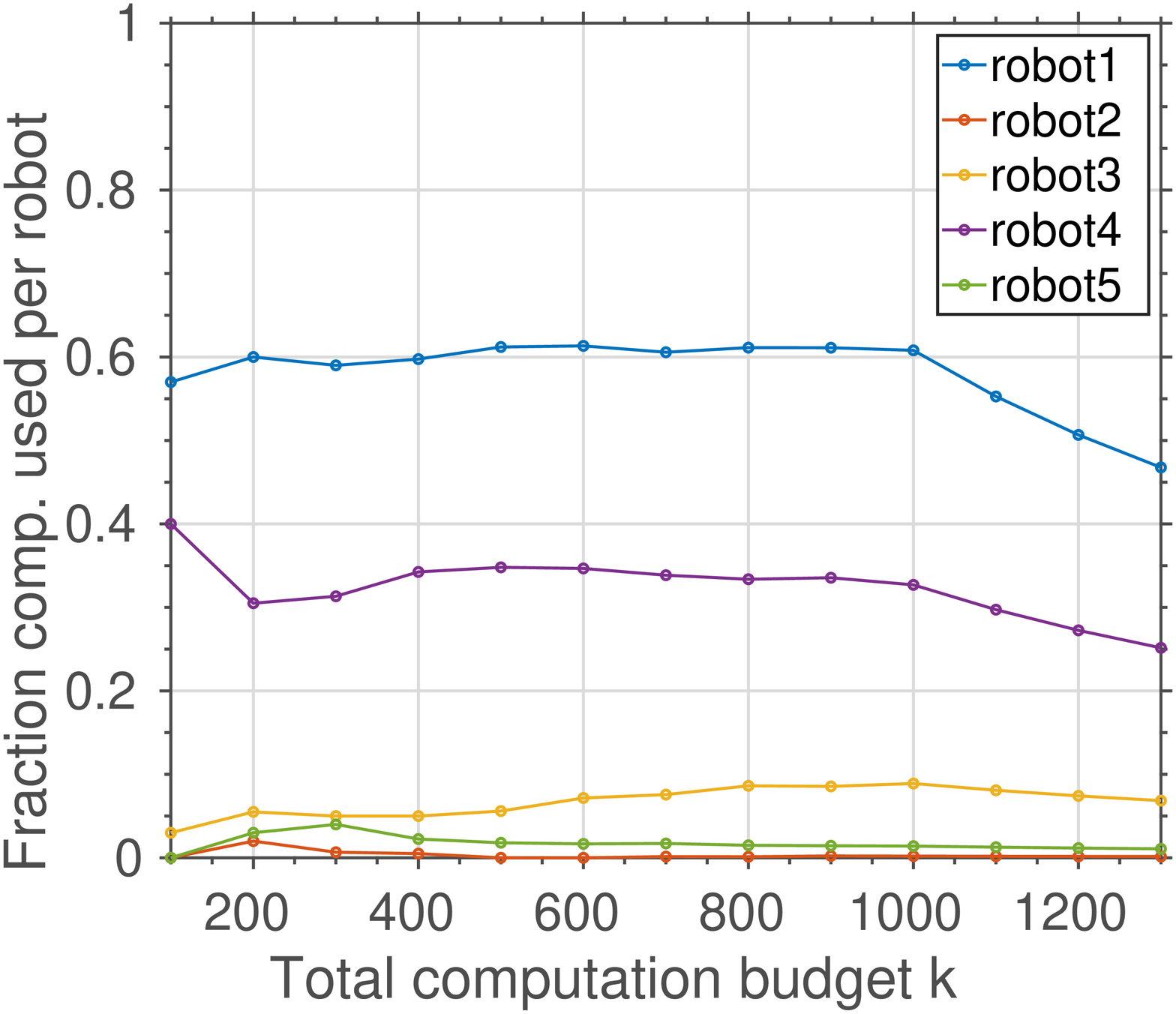}
\caption{\textsc{\small Modular-Greedy} \IndivUni{} ($b=20$MB)}
\label{fig:NLC_IU_comp_labor}
\end{subfigure}
\\
\begin{subfigure}[t]{0.30\textwidth}
\centering
\includegraphics[width=\textwidth]{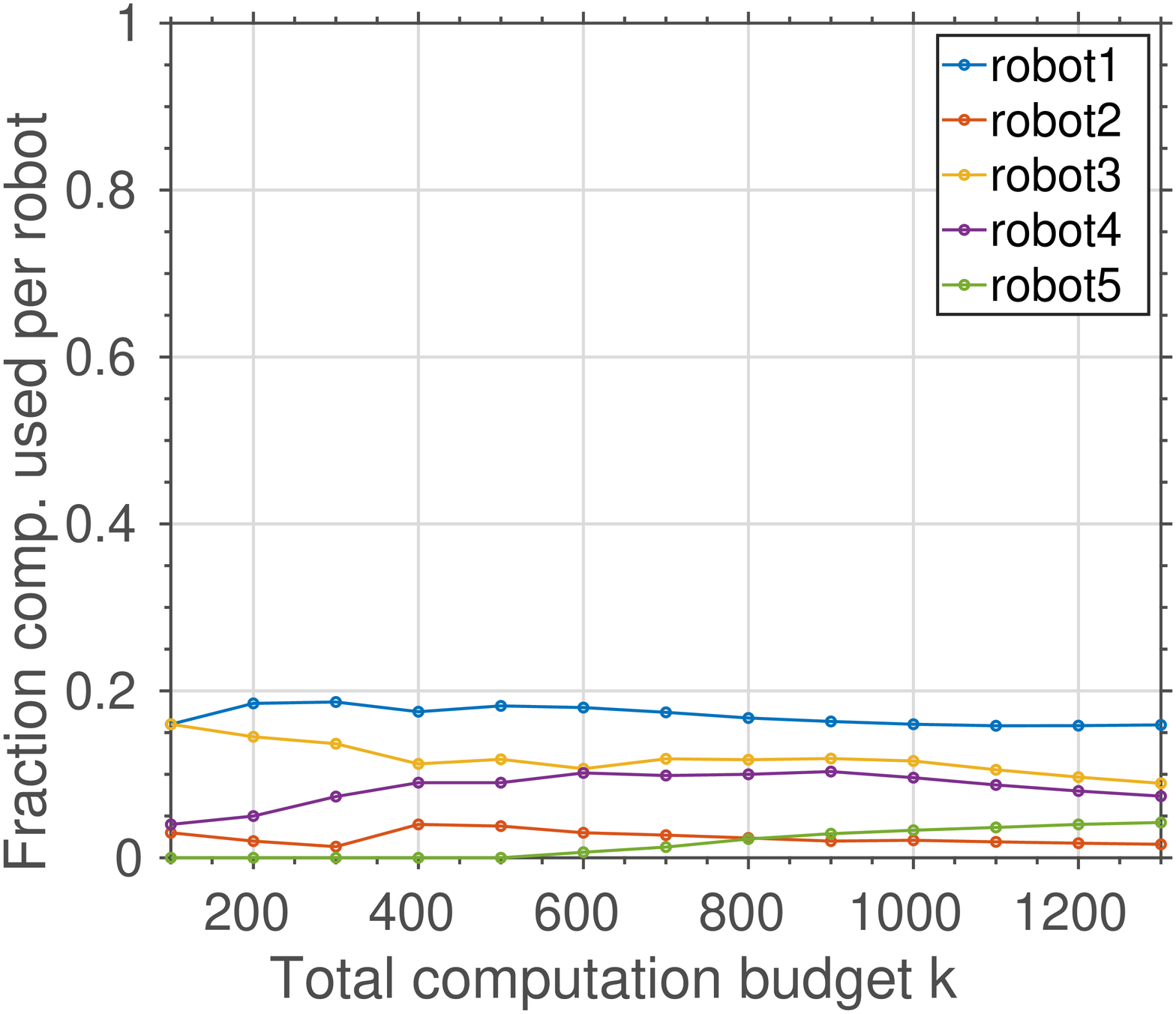}
\caption{\textsc{\small Modular-Greedy} \IndivUni{} ($b=20$MB) with individual computation budgets }
\label{fig:NLC_IUIK_comp_labor}
\end{subfigure}
\\
\vspace{0.2cm}
\begin{subfigure}[t]{0.30\textwidth}
\centering
\includegraphics[width=\textwidth]{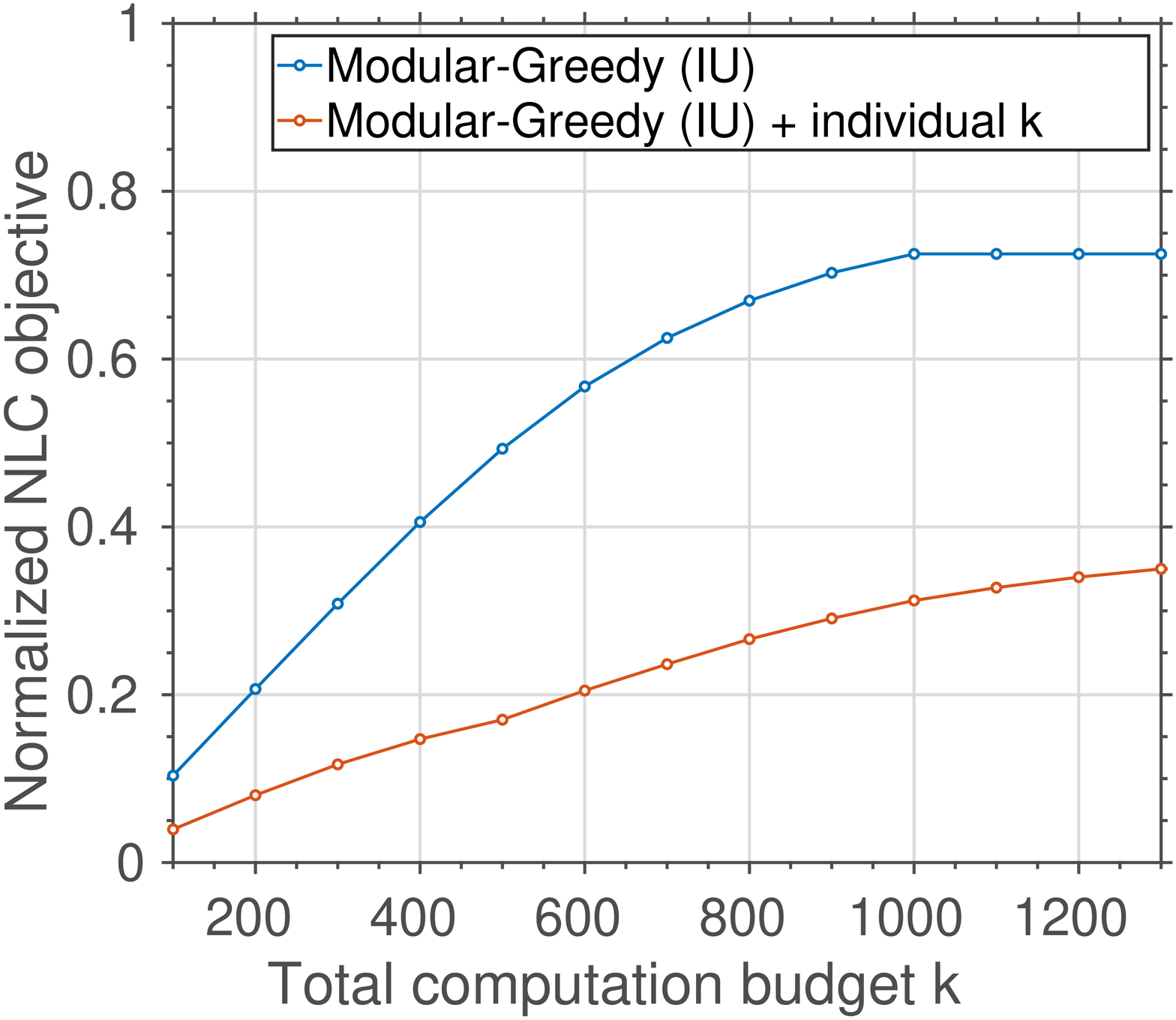}
\caption{Normalized $f_\text{NLC}$ objective}
\label{fig:NLC_IU_vs_IUIK}
\end{subfigure}
\caption{\small 
Induced individual computations
in the solutions of \textsc{\small Modular-Greedy} under \IndivUni{}.
Each robot has a fixed individual communication budget of $b_i=4$MB,
while the total computation budget $k$ varies from 100 to 1300.
\textbf{(a)} Fraction of total computation budget used by each robot.
\textbf{(b)} Results after dividing $k$ equally among robots.
\textbf{(c)} While the addition of individual computation budgets significantly
improves the balance of the induced computation labor, as a trade-off it also 
limits the overall team performance as evaluated by the $f_\text{NLC}$ objective. 
}
\label{fig:NLC_comp_labor}
\end{figure}

In addition,
we also study the induced division of labor in terms of computation under the \IndivUni{} communication cost model.
Each robot has a fixed individual communication budget of $b_i = 4$MB,
while the total computation budget $k$ varies from 100 to 1300.
Figure~\ref{fig:NLC_IU_comp_labor}
shows the
fraction of total computation budget used by each robot after running \textsc{\small Modular-Greedy}.
Similar to what we observed earlier in the case of communication,
without regulating individual computations 
the division of computation labor could be
arbitrarily skewed.
For example,
in many instances robot~1 (blue) 
ends up verifying 60\%
of all potential loop closures selected by the team.
In contrast,
the approach presented in Section~\ref{sec:individualBudgets}
explicitly ensures that each of the five robots
can verify at most 20\% of all selected potential matches;
see Figure~\ref{fig:NLC_IUIK_comp_labor}.
While this gives us more control over the induced computation workloads,
as a trade-off it also 
limits the overall team performance as evaluated by the $f_\text{NLC}$ objective;
see Figure~\ref{fig:NLC_IU_vs_IUIK}.

\subsection{Results with Submodular Objectives}

\begin{figure}[t]
\centering
\begin{subfigure}[t]{0.33\textwidth}
\centering
\includegraphics[width=\textwidth]{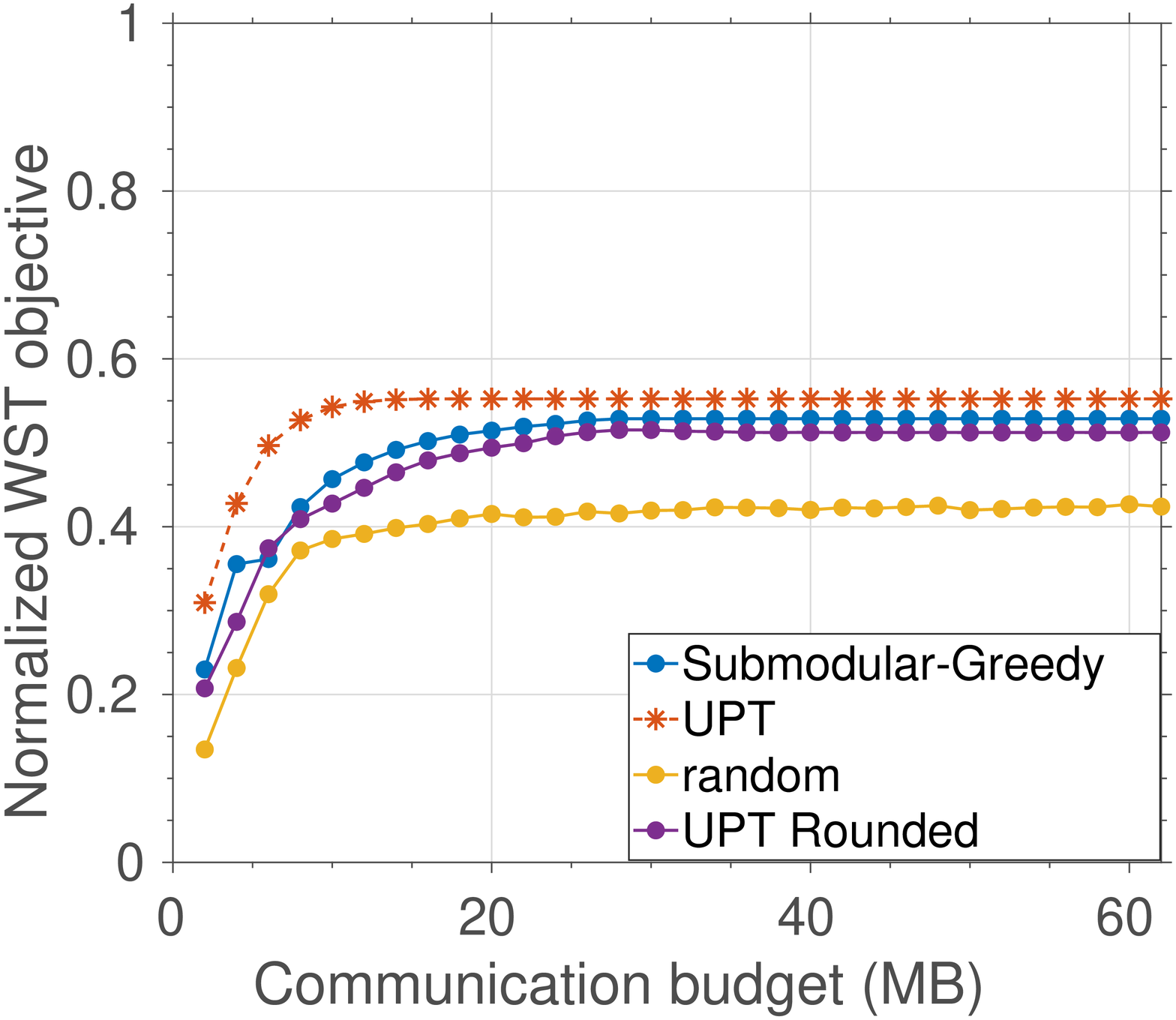}
\caption{KITTI00 $k=300$}
\label{fig:KITTI00_wst_k300}
\end{subfigure}
\\
\begin{subfigure}[t]{0.33\textwidth}
\centering
\includegraphics[width=\textwidth]{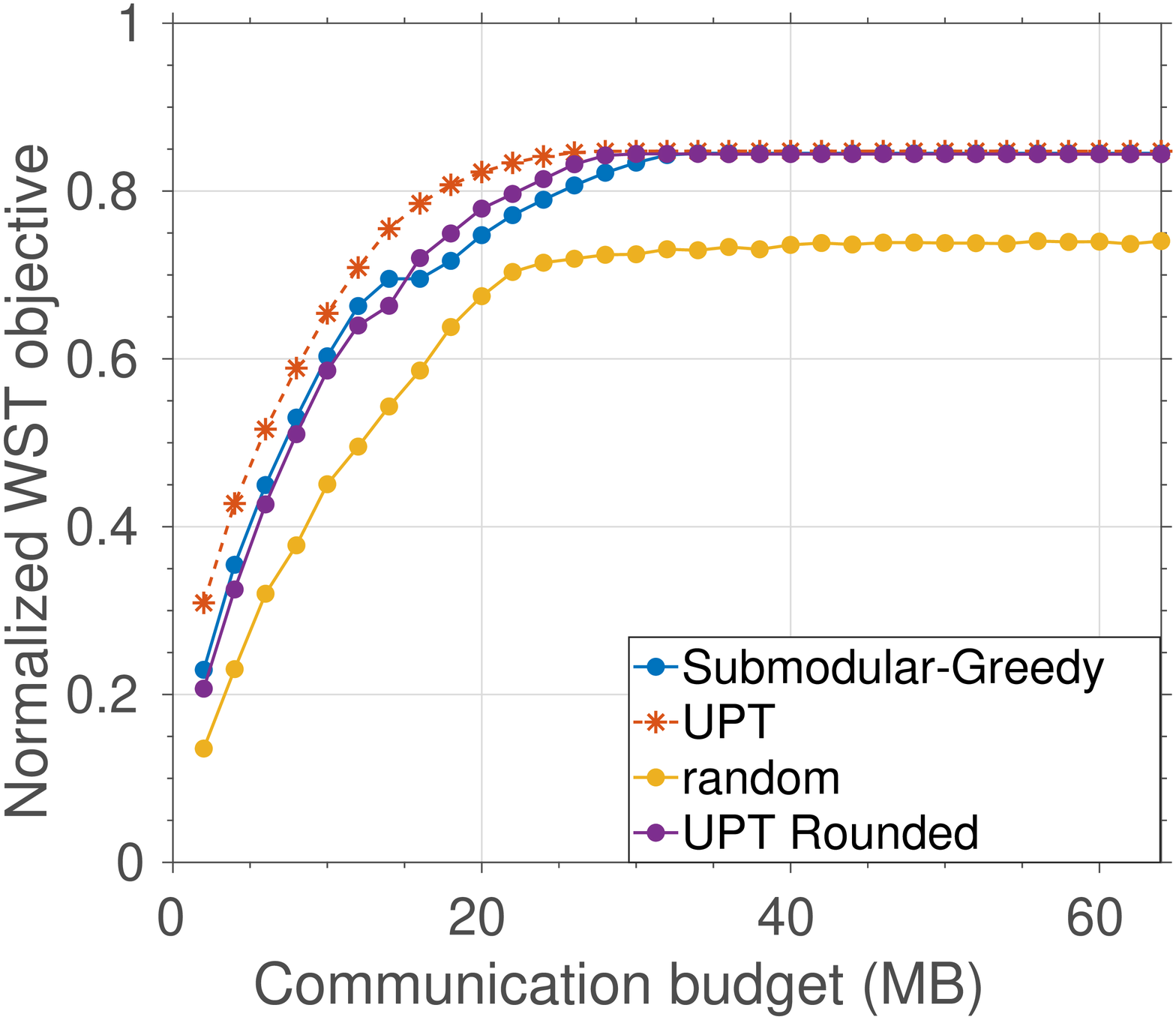}
\caption{KITTI00 $k=800$}
\label{fig:KITTI00_wst_k800}
\end{subfigure}
\\
\begin{subfigure}[t]{0.33\textwidth}
\centering
\includegraphics[width=\textwidth]{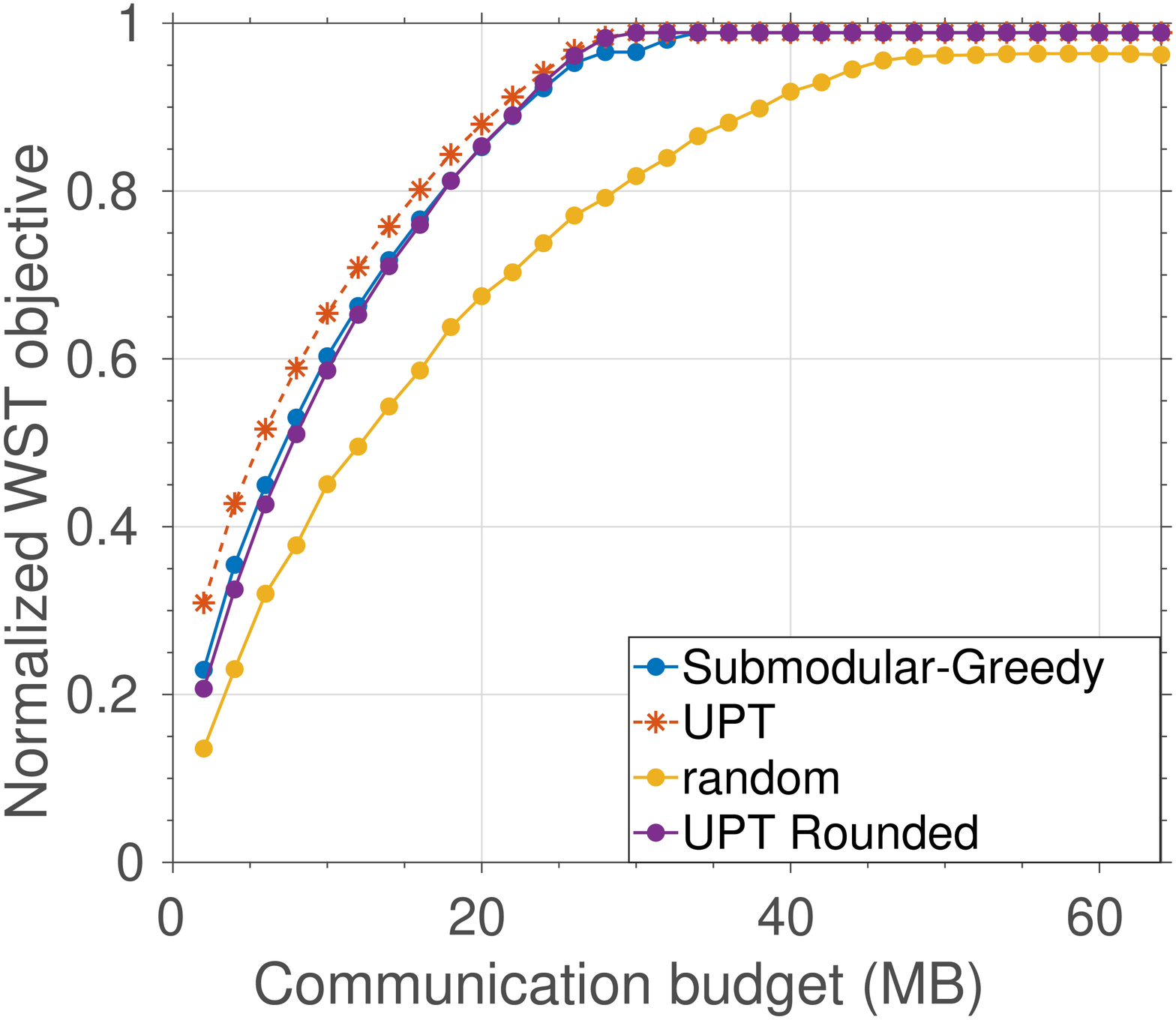}
\caption{KITTI00 $k=1300$}
\label{fig:KITTI00_wst_k1300}
\end{subfigure}
\caption{\small Performance of \textsc{Submodular-Greedy}
under the \TotalUni{} communication cost model,
using the tree connectivity (WST) as objective.
Each figure uses a fixed computation budget ($k=300,800, 1300$). 
The horizontal axis shows the total communication budget $b$ for \TotalUni{},
varying from $2$MB to $62$MB.
The vertical axis shows the normalized WST objective.
The submodular greedy solution is compared with the upper bound from 
convex relaxation (UPT),
the rounded solution from the upper bound (UPT rounded), as well as the solution
returned by the random algorithm.
}
\label{fig:KITTI00_wst}
\end{figure}

Figures~\ref{fig:KITTI00_wst_k300}-\ref{fig:KITTI00_wst_k1300}
show the performance of \textsc{\small Submodular-Greedy}
on the same KITTI~00 exchange graph 
under the \TotalUni{} regime. 
The
tree connectivity performance metric \eqref{eq:WST}
is used as the objective function.  
Similar to the modular experiments, 
each plot uses a fixed computation budget $k$, 
and a total communication budget varying 
from $2$MB-$62$MB.
The proposed algorithm is compared with the baseline random
algorithm.
As a surrogate for the optimal solution,
we compute the upper bound (UPT) using the convex relation approach
described in Section~\ref{sec:cvx}.
The fractional solution returned together with UPT is 
further rounded to obtain an
integer solution (UPT Rounded). 
To guarantee feasibility, 
$b$ vertices are first selected greedily 
using the values of the relaxed indicator variables 
$\ppp \triangleq [\pi_1,\dots,\pi_{n}]^\top \in [0,1]^n$.
Subsequently, $k$ edges are selected greedily from all edges 
covered by the selected vertices using the values of the relaxed indicator variables
$\bell  \triangleq [\ell_1,\dots,\ell_m]^\top \in [0,1]^m$. 
All objective values shown are 
normalized by the maximum 
achievable objective given infinite budgets. 

In all instances, 
the performance of \textsc{\small Submodular-Greedy} clearly
outperforms the random baseline, 
and is furthermore close to UPT.
These results
empirically validate the theoretical performance guarantees proved in 
Section~\ref{sec:submodular_general}.
The inflection point of each \textsc{\small Submodular-Greedy} curve
corresponds to the point where the algorithm switches from greedily selecting
edges (\textsc{\small Edge-Greedy})
to greedily selecting vertices (\textsc{\small Vertex-Greedy}).
Near inflection points, sometimes UPT Rounded
achieves a higher performance than \textsc{\small Submodular-Greedy}. This
observation is consistent with what we expected from Figure~\ref{fig:sgrd_apx}.
Note that this rounding lacks any worst-case performance guarantees.
In addition, obtaining the rounded solution requires solving
a potentially large MAXDET problem which 
may not be an option on computationally constrained platforms.

\subsection{Cross-Objective Analysis}
\label{sec:ate_experiments}

\begin{figure}[t]
\centering
\begin{subfigure}[t]{0.33\textwidth}
\centering
\includegraphics[width=\textwidth]{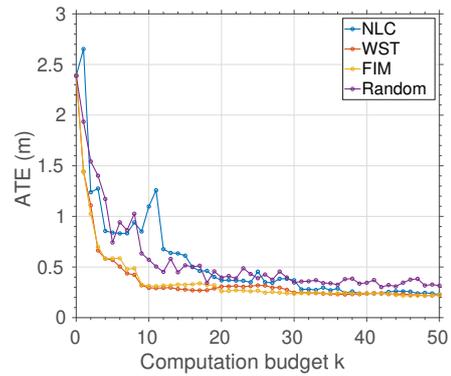}
\caption{Absolute trajectory error (ATE)}
\label{fig:atlas_ate}
\end{subfigure}
\\
\begin{subfigure}[t]{0.33\textwidth}
\centering
\includegraphics[width=\textwidth]{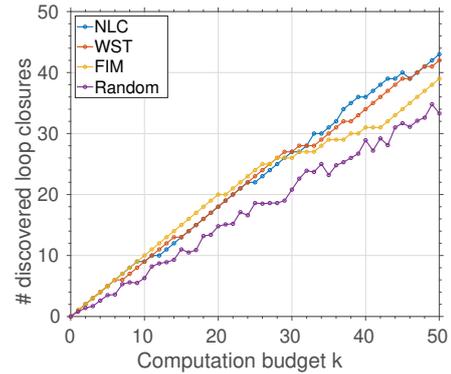}
\caption{Number of discovered loop closures}
\label{fig:atlas_tc}
\end{subfigure}
\caption{Cross-objective analysis in simulation
on \textbf{(a)} absolute trajectory error (ATE)
and \textbf{(b)} number of discovered loop closures. A
fixed communication budget of $2.5$MB is imposed under \TotalUni{}.}
\end{figure}

To complete our experimental analysis, 
we perform cross evaluations of the proposed 
performance metrics ($f_\text{FIM}$,
$f_\text{WST}$,
and $f_\text{NLC}$).
In these experiments,
all pose-graph SLAM
instances are solved using g2o \citep{kummerle2011g}. 
Figure~\ref{fig:atlas_ate}
shows the evolution of 
absolute trajectory error (ATE) 
as a function of the computation budget $k$ in simulation.
As the ATE converges quickly, 
we use a range of small computation budgets ($k=1$ to $k=50$),
with a fixed communication budget of $b=2.5$MB in \TotalUni{}.
In general, using the proposed metrics leads to faster decrease of ATE
compared to randomly selecting potential loop closures. 
Moreover, note that the
estimation-theoretic objectives $f_\text{FIM}$ and $f_\text{WST}$
outperform $f_\text{NLC}$.
Such results are expected, as
these objectives 
inherently favor edges
that lead to more reduction in 
estimation uncertainty.
Lastly,
Figure~\ref{fig:atlas_tc} shows the number of established loop closures 
as a function of $\kcomp$.
Once again, 
the proposed performance metrics 
outperform the random baseline noticeably.

\section{Conclusion}
\label{sec:conclusion}

Inter-robot loop closures constitute the backbone of CSLAM systems needed 
for multi-robot navigation in GPS-denied environments. 
In real-world scenarios, detecting inter-robot loop closures is a resource-intensive process with a
rapidly growing search space.
This task thus can be extremely challenging for robots subject to various operational and
resource constraints.
It is thus crucial for robots to be aware of these constraints, and to actively adapt to them by
intelligently utilizing their limited mission-critical resources.

We presented such a \emph{resource-aware} framework for distributed inter-robot
loop closure detection under budgeted communication and computation.
In particular, we sought an exchange-and-verification policy that maximizes a
monotone submodular performance metric under resource constraints.
Such a policy determines (i) ``which observations should be shared with the
team'', and (ii) ``which subset of potential matches is worthy of geometric
verification''. This problem is NP-hard in general.
As our main contribution, we provided efficient greedy approximation algorithms
with provable performance guarantees.
In particular, for monotone \emph{modular} performance metrics such as the
expected number of true loop closures,
the proposed algorithms achieve constant-factor approximation ratios
under multiple communication and computation cost models.
In addition, we also proposed an approximation algorithm for general monotone submodular
performance metrics with a performance guarantee that varies with the ratio of
resource budgets, as well as the extent of perceptual ambiguity.
The proposed framework
was extensively validated on real and synthetic benchmarks,
and empirical near-optimal performance was demonstrated via a natural convex relaxation scheme.

It remains an open problem whether constant-factor approximation 
for any monotone submodular objective is possible. We plan to study this open problem as
part of our future work.
Additionally, although the burden of verifying potential loop closures is
distributed among the robots, our current framework still relies on a
centralized scheme for evaluating the performance metric and running the
approximation algorithms.
In the future, we plan to leverage recent results on distributed submodular
maximization to eliminate such reliance on centralized computation.

\subsubsection*{Acknowledgments} 
This work was supported in part by the NASA Convergent
Aeronautics Solutions project Design Environment for Novel
Vertical Lift Vehicles (DELIVER), by ONR under BRC award N000141712072, and by
ARL DCIST under Cooperative Agreement Number W911NF-17-2-0181.

\bibliographystyle{SageH}
\bibliography{slam}

\onecolumn
\appendix
\section{Proofs}
\label{sec:app}

\subsection{Proof of Theorem~\ref{th:fvNMS} and Theorem~\ref{th:fvNMSPairwise}}
\label{pf:fvNMS}
In this section, we give the proofs for theorems presented in
Section~\ref{sec:modular}.
First, we establish the following more general result 
which holds for any partitioning of the edge set $\Eall$.
Both theorems in this section then
follow as special cases.
While these results are established independently \citep{Tian18_WAFR},
we note that an even more general case with an arbitrary matroid constraint is presented in 
the seminal work of \citep[Proposition~3.1]{nemhauser1978analysis}.

\begin{theorem}
\normalfont
Let $\Eall = \Ecal_1 \uplus \cdots \uplus \Ecal
_{n_b}$.
Define $\he: 2^{\Eall} \to \mathbb{R}_{\geq 0}$ as follows,
\begin{equation}
	\begin{aligned}
	  \he(\Acal) \triangleq \; &\underset{\Ecal \subseteq \Acal}{\text{maximize}}
	  & & \sum_{e \in \Ecal} p(e), \\
	  & \text{subject to}
	  & & |\Ecal \cap \Ecal_i| \leq k_i,\;i \in [n_b].
	\end{aligned}
	\label{eq:hedef}
\end{equation}
where $k_i \geq 0,\;i \in [n_b]$.
Furthermore, define $\fv: 2^{\Vall} \to \mathbb{R}_{\geq 0}$
as, 
\begin{equation}
\fv(\Vcal) \triangleq \he(\edg(\Vcal)).
\end{equation}
Then $\fv$ is NMS.
\label{thm:heNMS}
\end{theorem} 

\begin{proof}[\normalfont \bfseries Proof of Theorem~\ref{thm:heNMS}]
We first show that $\he$ is NMS.
\begin{itemize}
		\item[$\diamond$] Normalized: $\he(\varnothing) = 0$ by definition.
		\item[$\diamond$] Monotone:
		Let $\Acal \subseteq \Bcal \subseteq \Eall$. 
		$\he(\Acal) \leq \he(\Bcal)$ follows again from the definition.
		\item[$\diamond$] Submodular:
		Let $\Acal \subseteq \Bcal \subseteq \Eall$, and $e \in \Eall \setminus \Bcal$. We intend to show,
		\begin{equation}
		\he(\Acal \cup \{e\}) - \he(\Acal) \geq \he(\Bcal \cup \{e\}) - \he(\Bcal).
		\label{eq:submodular_def}
		\end{equation}
		If $\he(\Bcal \cup \{e\}) = \he(\Bcal)$, \eqref{eq:submodular_def}
		follows from the monotonicity of $\he$. 
		We thus focus on cases
		where $\he(\Bcal \cup \{e\}) > \he(\Bcal)$. 
		Let $\Ecal_j$ be the partition that $e$ belongs to, i.e., $e \in \Ecal_j$.
		Furthermore, define $\Acal_j \triangleq \Acal \cap \Ecal_j$
		and $\Bcal_j \triangleq \Bcal \cap \Ecal_j$. 
		Clearly, we have $\Acal_j \subseteq \Bcal_j$.
		Given any subset of edges $\Ccal \subseteq \Eall$, define
		$f_k(\Ccal)$ as follows,
		\begin{equation}
		f_k(\Ccal) \triangleq
		\begin{cases}
		\text{$k$th largest probability in $\Ccal$},  &\text{if } |\Ccal| \geq k,\\
		0,  &\text{otherwise.}
		\end{cases}
		\end{equation}
		As $\Acal_j \subseteq \Bcal_j$, it follows that 
		$f_k(\Acal_j) \leq f_k(\Bcal_j)$ for all $k \in \mathbb{N}$.
		Since
		$\he(\Bcal \cup \{e\}) > \he(\Bcal)$,
		it must hold that $p(e) \geq f_{k_j}(\Bcal_j) \geq f_{k_j}(\Acal_j)$.
		Thus,
		\begin{equation}
		\he(\Acal \cup \{e\}) - \he(\Acal)
		\geq 
		p(e) - f_{k_j}(\Acal_j)
		\geq
		p(e) - f_{k_j}(\Bcal_j)
		= 
		\he(\Bcal \cup \{e\}) - \he(\Bcal).
		\end{equation}
\end{itemize}
By Theorem~\ref{thm:fe2fv}, $\fv$ is also NMS,
which concludes the proof; see \citep{tian18}.
\qed
\end{proof}

\begin{proof}[\normalfont \bfseries Proof of Theorem~\ref{th:fvNMS}]
This is a special case of Theorem~\ref{thm:heNMS} with a single block (entire $\Eall$).
\end{proof}

\begin{proof}[\normalfont \bfseries Proof of Theorem~\ref{th:fvNMSPairwise}]
This is a special case of Theorem~\ref{thm:heNMS}
with $\displaystyle \Eall = \uplus_{\substack{i,j \in [r] \\ j > i}} \Ecal_{ij}$,
where $i,j$ correspond to robot indices; see Section~\ref{sec:individualBudgets}.
\end{proof}

\subsection{Proof of Theorem~\ref{thm:submodular_apx}}
\label{pf:submodular_apx}
Before proving the main theorem,
we first establish approximation guarantees
for the individual components of
\textsc{Submodular-Greedy},
namely 
\textsc{\small Edge-Greedy} (Lemma~\ref{lem:egrd_apx})
and 
\textsc{\small Vertex-Greedy} (Lemma~\ref{lem:vgrd_apx}).

\begin{lemma}
\normalfont
\textsc{\small Edge-Greedy} (Algorithm~\ref{alg:egreedy}) is an
$\alpha_e(\kcomm, \kcomp)$-approximation algorithm for Problem~\ref{prob:codesign} under \TotalUni{}, where
\begin{align}
		\alpha_e(\kcomm, \kcomp) \triangleq 1-\exp \big (-\min\,\{1,\kcomm/\kcomp\} \big ).
		\label{eq:egrd_apx1}
\end{align}
\label{lem:egrd_apx}
\end{lemma}
\begin{proof}[\normalfont \bfseries Proof of Lemma~\ref{lem:egrd_apx}]

As shown in Algorithm~\ref{alg:egreedy}, we exit the standard greedy loop (line~\ref{alg:egrd_1_start}-\ref{alg:egrd_1_end}) whenever the next selected edge violates 
either the communication or the computation budget. 
In addition, the heuristic optimization at the end (line~\ref{alg:egrd_2_start}-\ref{alg:egrd_2_end}) seeks to improve the solution while remaining feasible.
Therefore, the returned solution from \textsc{\small Edge-Greedy} is guaranteed to be feasible.

Now, we establish the performance guarantee presented in the lemma. To do so, we only look at 
the solution after the standard greedy loop (line~\ref{alg:egrd_1_start}-\ref{alg:egrd_1_end}). 
Let $\OPTe$ denote the optimal value of Problem~\ref{prob:codesign}.
Consider the relaxed version of Problem~\ref{prob:codesign} where we remove the communication budget.
In Section~\ref{sec:infinite_b}, we have shown that the reduced problem is,
\begin{equation}
	\underset{\Ecal \subseteq \Eall}{\text{maximize }} \fe(\Ecal) \text{ s.t. } |\Ecal| \leq \kcomp.
	\label{prob:ijrr_pe}
\end{equation}
Let OPT$_e$ denote the optimal value of the relaxed problem. Clearly, OPT$_e$ $\geq \OPTe$.
Let $\Egrd$ be the set of selected edges after line~\ref{alg:egrd_1_end}. 
Note that $\Egrd$ contains \emph{at least} $\min(\kcomm, \kcomp)$ edges.
Let $\Egrd' \subseteq \Egrd$ be the set formed by the \emph{first} $\min(\kcomm, \kcomp)$ edges selected by \textsc{\small Edge-Greedy}.
Note that $\Egrd'$ is also the solution of running the classic greedy algorithm \citep{nemhauser1978analysis} on \eqref{prob:ijrr_pe}
for $\min(\kcomm, \kcomp)$ iterations.
It thus follows that,
\begin{align}
	\fe(\Egrd) 
	& \geq \fe(\Egrd') && \text{(monotonicity of $\fe$)} \\
	& \geq \Big ( 1-\exp \big (- \min\{\kcomm, \kcomp\}/\kcomp \big ) \Big) \cdot \text{\text{OPT}$_e$} && \text{\cite[Theorem~$1.5$]{krauseSurvey}} \\
	& = \alpha_e(\kcomm, \kcomp) \cdot  \text{\text{OPT}$_e$}  && \text{(def. of $\alpha_e$)}\\
	& \geq \alpha_e(\kcomm, \kcomp) \cdot \OPTe. && \text{(\text{OPT}$_e \geq \OPTe$)} 
\end{align}

\qed
\end{proof}


\begin{lemma}
\normalfont
Let $\Delta$ be the maximum vertex degree in $\Gcal$. \textsc{\small
Vertex-Greedy} (Algorithm~\ref{alg:vgreedy}) is an $\alpha_v(\kcomm, \kcomp,
\Delta)$-approximation algorithm for Problem~\ref{prob:codesign} under \TotalUni{}, where
\begin{align}
	\alpha_v(\kcomm, \kcomp, \Delta) & \triangleq 1-\exp \big (-\min\,\{1, \lfloor \kcomp / \Delta \rfloor/\kcomm \} \big).
	\label{eq:vgrd_apx1}
\end{align}
\label{lem:vgrd_apx}
\end{lemma}
\begin{proof}[\normalfont \bfseries Proof of Lemma~\ref{lem:vgrd_apx}]
Similar to \textsc{\small Edge-Greedy},
we terminate the greedy loop whenever the next selected vertex violates 
either the communication or the computation budget. 
Thus, the returned solution is guaranteed to be feasible. 
Let $\OPTe$
denote the optimal value of Problem~\ref{prob:codesign}.
Consider the relaxed version of Problem~\ref{prob:codesign} under \TotalUni{} where we remove the computation budget.
In Section~\ref{sec:infinite_k}, we have shown that the relaxed problem is equivalent to the following,
\begin{equation}
	\underset{\Vcal \subseteq \Vall}{\text{maximize }} \fv_\text{com}(\Vcal) \text{ s.t. } |\Vcal| \leq \kcomm.
	\label{prob:rss_pv}
\end{equation}
where $\fv_\text{com}: 2^{\Vall} \to \Rset_{\geq 0}$ is defined according to \eqref{eq:fvdef}.
Let OPT$_v$ denote the optimal value of 
\eqref{prob:rss_pv}.
Clearly, OPT$_v$ $\geq \OPTe$.
Let $\Egrd, \Vgrd$ be the selected edges and vertices after running \textsc{\small Vertex-Greedy} on Problem~\ref{prob:codesign}.
Note that $\Vgrd$ contains \emph{at least}
$\min(\kcomm, \lfloor \kcomp / \Delta \rfloor)$ vertices.
This is because whenever we select less than this number of vertices, there is guaranteed to be enough computation and communication budgets to select the next vertex and include all its incident edges.
Now, let $\Vgrd' \subseteq \Vgrd$ contain the \emph{first} $\min(\kcomm, \lfloor \kcomp / \Delta \rfloor)$ vertices selected by \textsc{\small Vertex-Greedy}.
Note that $\Vgrd'$ is also the solution of running the classic greedy algorithm \citep{nemhauser1978analysis} on \eqref{prob:rss_pv} for $\min(\kcomm, \lfloor \kcomp / \Delta \rfloor)$ iterations.
It thus follows that,
\begin{align}
	\fe(\Egrd) & = \fv_\text{com}(\Vgrd)  && \text{(def. of $\fv_\text{com}$)} \\
	& \geq \fv_\text{com}(\Vgrd') && \text{(monotonicity of $\fv_\text{com}$)} \\
	& \geq \Big ( 1-\exp \big (- \frac{\min\{\kcomm, \lfloor \kcomp / \Delta \rfloor\}}{\kcomm} \big ) \Big) \cdot \text{\text{OPT}$_v$} && \text{\cite[Theorem~$1.5$]{krauseSurvey}}\\
	& = \alpha_v(\kcomm, \kcomp, \Delta) \cdot \text{\text{OPT}$_v$} && \text{(def. of $\alpha_v$)} \\
	& \geq \alpha_v(\kcomm, \kcomp, \Delta) \cdot \OPTe. && \text{(\text{OPT}$_v \geq \OPTe$)} 
\end{align}

\qed
\end{proof}

\noindent
Having established Lemma~\ref{lem:egrd_apx}
and Lemma~\ref{lem:vgrd_apx}, 
we now give a straightforward proof of the main theorem.
	
\begin{proof}[\normalfont \bfseries Proof of Theorem~\ref{thm:submodular_apx}]
By Lemma~\ref{lem:egrd_apx} and Lemma~\ref{lem:vgrd_apx},
\begin{align}
	\alpha(\kcomm, \kcomp, \Delta) &= \max \, \{\alpha_e(\kcomm, \kcomp), \alpha_v(\kcomm, \kcomp, \Delta) \} \\
	& = 1 - \exp \big (-\min \big \{ 1, \max \, ( \kcomm / \kcomp, \lfloor \kcomp / \Delta \rfloor / \kcomm) \big \} \big) \\
	& = 1 - \exp \big (-\min\{1, \gamma\} \big ).
\end{align}

\qed
\end{proof}

\end{document}